\documentclass[twoside]{article}

\usepackage[preprint]{aistats2027}

\usepackage[numbers]{natbib}

\usepackage[utf8]{inputenc} \usepackage[T1]{fontenc}    \usepackage{hyperref}       \usepackage{url}            \usepackage{booktabs}       \usepackage{amsfonts}       \usepackage{nicefrac}       \usepackage{microtype}      \usepackage{xcolor}         \usepackage{tikz}

\usepackage{amssymb,amsmath,amsfonts,amsthm,commath}
\usepackage[algo2e,ruled]{algorithm2e}
\usepackage{multirow}
\usepackage[capitalise]{cleveref}
\usepackage{enumitem}
\usepackage{wrapfig}
\usepackage{colortbl}

\theoremstyle{plain}
\newtheorem{theorem}{Theorem}
\newtheorem{proposition}{Proposition}
\newtheorem{lemma}{Lemma}
\newtheorem{corollary}{Corollary}
\theoremstyle{definition}
\newtheorem{definition}{Definition}

\theoremstyle{remark}
\newtheorem{remark}{Remark}

\usepackage{subcaption}

\newcommand{\op}[1]{\operatorname{#1}}
\newcommand{\C}[1]{{\mathcal{#1}}} \newcommand{\B}[1]{{\mathbb{#1}}}

\begin{document}

\twocolumn[

\aistatstitle{The Tilted Sampling Problem: Optimism, Posterior Sampling, and Logarithmic Regret}

\aistatsauthor{ Mohammad Pedramfar \And Siamak Ravanbakhsh }

\aistatsaddress{ Mila -- Qu\'ebec AI Institute, McGill University\\
\texttt{\{mohammad.pedramfar,siamak.ravanbakhsh\}@mila.quebec}
} ]

\begin{abstract}
Many modern learning systems are used not to select a single best action, but to sample diverse high-quality outputs. Examples include entropy-regularized reinforcement learning, RLHF, and reward-guided generative modeling. Motivated by these settings, we study \emph{tilted sampling}: an online learning problem in which the goal is to sample from the Gibbs distribution induced by an unknown reward function.
We define sampling regret as the temperature-normalized cumulative KL divergence between the learner's distribution and the Gibbs target, corresponding to the suboptimality gap in entropy-regularized reward maximization. Our framework generalizes recent sampling-based and KL-regularized formulations in bandits and reinforcement learning.
Our main contribution is a general regret framework that captures the different uncertainty structure of tilted sampling and classical regret minimization. While standard regret scales as $\widetilde{O}(\sqrt{T})$ due to first-order uncertainty, sampling regret is governed by squared confidence widths, yielding logarithmic-in-$T$ regret in standard settings at fixed temperature.
We develop sampling analogues of UCB and Thompson sampling (UCS and TS$^2$) and provide unified regret decompositions that explain and extend recent logarithmic guarantees, including the first Bayesian sampling-regret decomposition based on new bounds for posterior-random Gibbs distributions. Our results show that learning to sample from a soft-optimal distribution is statistically distinct from learning to identify a hard optimum.
\end{abstract}

\section{Introduction}

Classical bandit and reinforcement learning problems are formulated around \emph{optimization}: the goal is to identify and repeatedly select actions with maximal reward \citep{lai1985asymptotically, auer02_finit_analy_multiar_bandit_probl, sutton18_reinf_learn}. Given an unknown reward function $f_\theta$, performance is measured by regret relative to a maximizer $\arg\max_a f_\theta(a)$. This formulation is appropriate when the desired behavior is a single best decision.

Many modern learning systems, however, are not used in this way. Large language models generate text by sampling \citep{ouyang2022training}, diffusion models generate images by sampling \citep{ho2020denoising}, and stochastic policies in reinforcement learning sample actions \citep{haarnoja18_soft_actor_critic}. In these settings, the output of the system is inherently distributional. The goal is often not to collapse onto one maximizer, but to bias a base distribution toward high-reward outputs while preserving diversity, entropy, or proximity to a reference model.

\begin{figure}[t]
  \centering
  \begin{subfigure}[b]{0.48\linewidth}
    \centering
    \includegraphics[width=\linewidth]{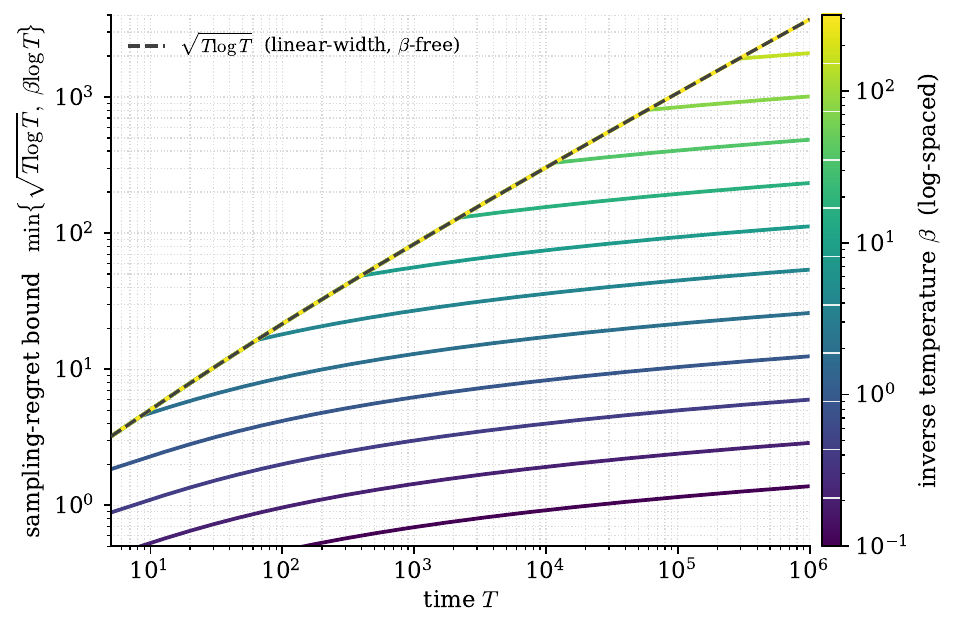}
    \caption{Bound vs.\ horizon $T$.}
    \label{fig:regret-vs-T}
  \end{subfigure}\hfill
  \begin{subfigure}[b]{0.48\linewidth}
    \centering
    \includegraphics[width=\linewidth]{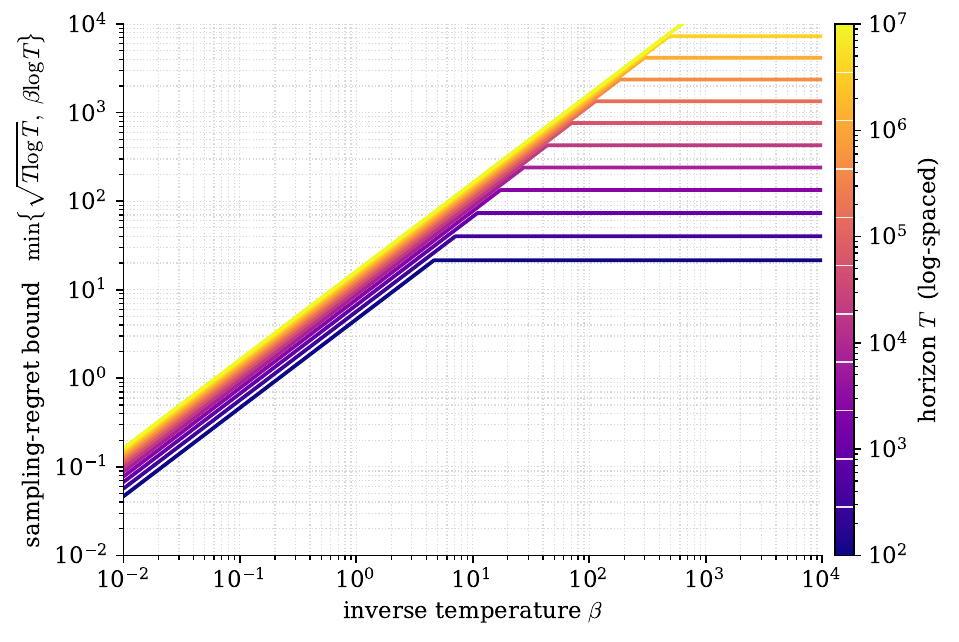}
    \caption{Bound vs.\ inverse temperature.}
    \label{fig:regret-vs-beta}
  \end{subfigure}\hfill
\caption{Schematic sampling-regret upper bound
    $R(T,\beta)=\min\!\bigl\{\sqrt{T\log T},\ \beta\log T\bigr\}$
    illustrating the complementary linear- and squared-width bounds for UCS
    and TS$^2$ in Theorems~\ref{thm:regret-decom-ucs}
    and~\ref{thm:regret-decom-ts2}.
    \textbf{(\subref{fig:regret-vs-T})}
    For fixed $\beta$ and $T>1$, the square-root branch is active when
    $T/\log T\leq\beta^2$, and the logarithmic branch when
    $T/\log T\geq\beta^2$.
    The large-horizon crossover satisfies
    $T^\star(\beta)=\beta^2\log T^\star(\beta)$, with
    $T^\star(\beta)\sim2\beta^2\log\beta$ as $\beta\to\infty$;
    beyond it the bound follows $\beta\log T$.
    \textbf{(\subref{fig:regret-vs-beta})}
    For each $T$ the bound is linear in $\beta$ up to the kink
    $\beta^{\star}(T)=\sqrt{T/\log T}$ and saturates at
    $\sqrt{T\log T}$ beyond it.
}
  \label{fig:regret-min-bounds}
\end{figure}

This structure appears explicitly in RLHF and preference-based fine-tuning, where the idealized KL-regularized policy has the form
\[
    \pi^*(a) \propto \pi_{\rm ref}(a)\exp(\beta r(a)).
\]
It also appears in entropy-regularized reinforcement learning and maximum-entropy control \citep{ziebart2008maximum, haarnoja2017reinforcement}, as well as reward-guided generative modeling \citep{dhariwal2021diffusion}. Abstractly, these settings lead to the Gibbs distribution
\[
    S_{\beta,t}(f_\theta)(da)
    =
    \frac{\exp(\beta f_\theta(a))}{Z_{\beta,t}(f_\theta)}\,\mu_t(da),
\]
where $\mu_t$ is a base measure on the available action set $\mathcal A_t$ and $\beta>0$ is an inverse-temperature parameter.

This paper studies the online learning problem induced by this target. The learner interacts with the environment and aims to sample from the unknown Gibbs distribution rather than identify a single optimal action. We call this problem \emph{tilted sampling}.

Our work builds directly on the multi-armed sampling framework of \citet{pedramfar26_multi_armed_sampl_probl_end_explor}, which introduced the finite-armed analogue of this problem and showed that sampling can require substantially less explicit exploration than optimization. We generalize this framework beyond finite arms to general action spaces, time-varying feasible sets, function classes, and Bayesian posterior sampling.

\paragraph{Sampling regret.}
We measure performance via the cumulative KL divergence between the learner's sampling distribution and the Gibbs target:
\[
    \mathrm{SRegret}(T)
    =
    \frac{1}{\beta}\sum_{t=1}^T
    D_{\rm KL}\!\left(
        S_{\beta,t}(g_t)
        \,\middle\|\,
        S_{\beta,t}(f_\theta)
    \right).
\]
This objective coincides with the suboptimality gap of entropy-regularized reward maximization \citep{ziebart2008maximum, haarnoja2017reinforcement}.

\paragraph{A different uncertainty structure.}
The main conceptual message is that learning to sample is statistically distinct from learning to optimize. Classical regret is driven by first-order uncertainty and scales as $\widetilde O(\sqrt{T})$ in standard models \citep{auer02_finit_analy_multiar_bandit_probl}.
In contrast, sampling regret can be controlled by second-order uncertainty through squared confidence widths. This leads to logarithmic-in-$T$ regret in many settings at fixed temperature. Such logarithmic guarantees have been established in KL-regularized bandits \citep{zhao2025logarithmic} and greedy RLHF analyses \citep{wu2025greedy}. Our contribution is a unified framework of confidence-width decompositions that explains and extends these results across sampling algorithms and function classes.

\paragraph{Bayesian regret.}
Beyond frequentist algorithms, we also introduce Bayesian sampling regret. This leads to TS$^2$, a Thompson-style sampling algorithm. Unlike classical Thompson sampling \citep{russo2018tutorial}, this requires controlling KL divergence between random Gibbs distributions. Its analysis requires new KL-divergence bounds for posterior-random Gibbs distributions and provides the first Bayesian regret decomposition for sampling objectives.

\paragraph{Algorithms and contributions.}
Our framework develops the tilted-sampling counterpart of the confidence-width approach to learning to optimize under bandit feedback \citep{russo14_learn_optim_via_poster_sampl}.

\emph{We provide the first general regret framework for learning soft-optimal (Gibbs) distributions under bandit feedback, unifying optimistic and greedy analyses and establishing a Bayesian sampling-regret decomposition.}

Our contributions are:
\begin{itemize}
    \item A general framework for tilted sampling beyond the finite-armed setting to time-varying action sets, function classes, and the Bayesian setting.
    \item We introduce Sampling With Confidence (SWC) framework, Upper Confidence Sampling (UCS), and Thompson Sampling for Sampling (TS$^2$).
    \item We show that the UCS and SWC regret decompositions generalize, simplify, and unify prior logarithmic-regret analyses for optimistic and greedy sampling methods through model-independent confidence-width bounds.
    \item We give a novel Bayesian sampling-regret decomposition for TS$^2$, based on new KL-divergence bounds for posterior-sampled Gibbs distributions.
    \item Our regret decompositions extend to both directions of KL, providing a unified analysis of forward and reverse KL-divergences. In particular, the UCS and TS$^2$ analyses yield simultaneous control of both quantities with identical bounds under the same confidence event and no exponential dependence on $\beta C$.
    \item We show that second-order uncertainty control yields logarithmic-in-time sampling regret in standard models for fixed temperature.
    \item We provide a black-box transfer of confidence-set and width-process analyses from UCB and Thompson sampling to tilted sampling under direct noisy reward feedback, yielding guarantees for the corresponding UCS and TS$^2$ procedures.
\end{itemize}

\section{Problem Setup and Regret Decomposition}

We study a sequential decision problem over $T$ rounds. 
Let $\mathcal{A}$ be the action space and let $\mathcal{F} = \{ f_\theta : \mathcal{A} \to \mathbb{R} \mid \theta \in \Omega \}$ be the parameterized objective class.
At each round $t$, the agent observes the available action set $\mathcal A_t \subseteq \mathcal A$ together with the history
\begin{align*}
H_t = (\mathcal A_1, A_1, Y_1, \dots, \mathcal A_{t-1}, A_{t-1}, Y_{t-1}),
\end{align*}
selects an action $A_t \in \mathcal A_t$, and receives
$Y_t = f_{\theta}(A_t) + \varepsilon_t$,
where $\theta \in \Omega$ is unknown, and $\mathbb E[\varepsilon_t \mid H_t,\mathcal A_t,A_t,\theta]=0$.
We assume $f(a)\in[0,C]$ for all $f \in \mathcal{F}$ and $a \in \mathcal A$.
In the Bayesian setting, we assume that $\theta$ is drawn from a prior distribution $\rho$ over the parameter space.

Let $A_t^\star \in \arg\max_{a\in\mathcal A_t} f_{\theta}(a)$ denote the optimal action.
Throughout, all argmax operations are assumed to use a fixed measurable tie-breaking rule.
The regret and Bayesian regret are
\begin{align*}
\mathrm{Regret}(T)
&= \textstyle \sum_{t=1}^T \big(f_\theta(A_t^\star)-f_\theta(A_t)\big), 
\\
\mathrm{BayesRegret}(T)
&= \mathbb E_{\theta \sim \rho}\!\left[\textstyle \sum_{t=1}^T \big(f_\theta(A_t^\star)-f_\theta(A_t)\big)\right],
\end{align*}
where the expectation in $\mathrm{BayesRegret}(T)$ is taken over the prior on $\theta$ as well as the randomness of the algorithm and observations.

\paragraph{Confidence sets.}
Let $\mathcal F_t = \mathcal F_t(H_t,\mathcal A_t) \subseteq \mathcal F$ be confidence sets constructed from the revealed information. Intuitively, $\mathcal F_t$ collects functions that are statistically plausible given the observed data and the sequence of available action sets. In many models, these arise from concentration inequalities or posterior credible regions (e.g., least-squares ellipsoids in linear models).

At each time step $t$, we define upper and lower confidence bounds only over the realized action set $\mathcal A_t$. 
More precisely, for $a \in \mathcal A_t$, let
$U_t(a) = \sup_{f\in\mathcal F_t} f(a)$
,
$L_t(a) = \inf_{f\in\mathcal F_t} f(a)$
, and
$w_{\mathcal F_t}(a) = U_t(a)-L_t(a).$

Confidence sets play a central role in both algorithm design and analysis. For UCB, the upper envelope $U_t$ induces optimism, encouraging exploration of actions with high uncertainty. For TS, while the algorithm itself does not explicitly use $\mathcal F_t$, these sets provide a convenient analytical tool to control regret. In particular, the width $w_{\mathcal F_t}(a)$ quantifies the uncertainty about the reward of action $a$, and bounding the cumulative width along the trajectory will be the key step in our regret analysis.

\paragraph{Algorithms.}
We consider two canonical approaches to balancing exploration and exploitation.
The \emph{Upper Confidence Bound (UCB)} algorithm selects actions optimistically according to the confidence sets:
\begin{align*}
A_t \in \arg\max_{a\in\mathcal A_t} U_t(a).
\end{align*}
This strategy favors actions with high upper confidence bounds, thereby directing exploration toward actions whose rewards are uncertain.
The \emph{Thompson Sampling (TS)} algorithm instead randomizes according to the posterior distribution. At each round, it samples $\tilde\theta_t \sim \rho(\cdot \mid H_t, \mathcal A_t)$ and selects
\begin{align*}
A_t \in \arg\max_{a\in\mathcal A_t} f_{\tilde\theta_t}(a).
\end{align*}
UCB explores via optimism, while TS explores via randomized models sampled from the current posterior.

\paragraph{Regret Decomposition}

Although these algorithms are operationally different—one deterministic and optimistic, the other randomized and Bayesian—their regret can be analyzed through a common lens. 
In particular, both admit bounds in terms of the confidence widths $w_{\mathcal F_t}(A_t)$, which capture the uncertainty encountered along the realized sequence of actions.
The following result is classical.~\cite{russo14_learn_optim_via_poster_sampl}
We include it here for completeness.

\begin{theorem}
\label{thm:ucb-ts}
Let $\mathcal E$ denote the event $\{ f_\theta \in \cap_{t=1}^T \mathcal F_t \}$, with the confidence sets and history generated by the algorithm under consideration.
For UCB with fixed unknown $\theta\in\Omega$, let $\mathbb E$ and $\mathbb P$ denote expectation and probability over observation noise and environmental or algorithmic randomness. Then
\[
\mathbb E[\mathrm{Regret}(T)]
\le
\mathbb E\!\left[\textstyle \sum_{t=1}^T w_{\mathcal F_t}(A_t)\right]
+
C T \mathbb P(\mathcal E^c).
\]
For Thompson sampling with $\theta\sim\rho$, let $\mathbb E$ and $\mathbb P$ denote expectation and probability under the joint law of the prior draw, observations, environment, and posterior-sampling algorithm. Then
\begin{align*}
\mathrm{BayesRegret}(T)
\le
\mathbb E\!\left[\textstyle \sum_{t=1}^T w_{\mathcal F_t}(A_t)\right]
+
2 C T \mathbb P(\mathcal E^c).
\end{align*}
\end{theorem}
See Appendix~\ref{apx:ucb-ts} for the proof.

\section{Tilted Sampling}
\label{sec:tilted-sampling}

We now introduce a distributional analogue of bandit optimization: the learner aims to match the Gibbs distribution induced by the unknown reward function. Measuring regret by KL divergence leads to bounds in terms of squared confidence widths. When these widths accumulate logarithmically, the resulting sampling-regret guarantees are logarithmic in $T$ for fixed temperature.

\subsection{Sampling Objective}

Because the available action set may vary with time, the relevant Gibbs distribution is defined on the currently revealed set $\mathcal A_t$ rather than on the ambient action space $\mathcal A$.
Fix $\beta > 0$, and let $\mu$ be a base probability measure on $\mathcal A$. For each round $t$, let $\mu_t$ denote the conditional measure
\begin{align*}
\mu_t(B) = \mu(B \mid \mathcal A_t) = \frac{\mu(B \cap \mathcal A_t)}{\mu(\mathcal A_t)},
\end{align*}
whenever $\mu(\mathcal A_t) > 0$. For any measurable function $f$, define the Gibbs distribution on $\mathcal A_t$
\begin{align*}
S_{\beta,t}(f)(da)
= \frac{1}{Z_{\beta,t}(f)} \exp(\beta f(a)) \, \mu_t(da),
\end{align*}
where
$Z_{\beta,t}(f) = \int_{\mathcal A_t} \exp(\beta f(a)) \, \mu_t(da)$.
Equivalently, $S_{\beta,t}(f)$ is the Gibbs tilt of the conditioned base measure $\mu_t$.

The target distribution which we wish to learn is $S_{\beta,t}(f_\theta)$, which concentrates mass on high-reward actions while preserving exploration through temperature $\beta^{-1}$.

\begin{remark}
To ensure that $\mu_t = \mu(\cdot \mid \mathcal A_t)$ is well defined, it suffices that each $\mathcal A_t$ be measurable and satisfy $\mu(\mathcal A_t) > 0$ almost surely. We assume throughout that these and any additional regularity conditions needed to define the Gibbs measures $S_{\beta,t}(f)$ hold. If $\mathcal A$ is finite or countable and $\mu$ has full support, these conditions are automatic.
\end{remark}

\subsection{Interaction Model}

At each round $t$, the learner selects a function $g_t : \mathcal{A}_t \to \mathbb{R}$ based on the history $H_t$ and the revealed action set $\mathcal{A}_t$, and samples
\begin{align*}
A_t \sim S_{\beta, t}(g_t).
\end{align*}
This contrasts with the standard bandit setting, where a single action is selected deterministically or via TS. Here, the learner controls a full distribution over actions.

Note that, while the goal is to match the target distribution induced by $f_\theta$, the learner may choose to act based on some $g_t \notin \C{F}$. 
However, as we will see in Upper Confidence Sampling (UCS) algorithm below, even if $g_t$ is not in $\C{F}$, it is beneficial for it to be close to functions in the confidence set $\mathcal{F}_t$ so that the action may provide more valuable information.

\subsection{Sampling Regret}

We measure performance via the divergence between the induced sampling distribution and the target distribution.
We define the (cumulative) sampling regret as
\begin{align*}
\mathrm{SRegret}(T)
=
\tfrac{1}{\beta} \textstyle \sum_{t=1}^T
D_{\mathrm{KL}}\big(S_{\beta,t}(g_t) \,\|\, S_{\beta,t}(f_\theta)\big),
\end{align*}
and the Bayesian sampling regret as
\begin{align*}
\mathrm{BSRegret}(T)
=
\tfrac{1}{\beta} \mathbb E\!\left[ \textstyle \sum_{t=1}^T
D_{\mathrm{KL}}\big(S_{\beta,t}(g_t) \,\|\, S_{\beta,t}(f_\theta)\big)
\right],
\end{align*}
where the expectation is taken over $\theta \sim \rho$ and all randomness in the learner and observations.
These notions of regret capture how well the learner approximates the optimal sampling distribution, rather than how well it identifies optimal actions.

Table~\ref{tbl:comparison} summarizes the key structural difference between bandit optimization and tilted sampling.

\begin{table}[t]
\centering
\small
\begin{tabular}{@{}ll@{}}
\toprule
\textbf{Bandit Optimization} & \textbf{Tilted Sampling} \\
\midrule
$\max f(a)$ & $\pi \propto \exp(\beta f)$ \\
Non-smooth & Smooth (KL) \\
First-order & Second-order \\
$\sum w_t \to \sqrt{T}$ & $\sum w_t^2 \to \log T$ \\
UCB / TS & UCS / TS$^2$ \\
\bottomrule
\end{tabular}
\caption{\textbf{Optimization vs.\ sampling.}
KL geometry smooths the objective, converting first-order uncertainty into second-order regret.}
\label{tbl:comparison}
\end{table}

\paragraph{Regularized-value interpretation.}
Sampling regret is the suboptimality gap of KL-regularized reward maximization. For a distribution $\pi$ on $\mathcal A_t$ with finite $D_{\mathrm{KL}}(\pi\|\mu_t)$, define
\[
J_{\beta,t,f}(\pi)
= \mathbb E_{A\sim\pi}[f(A)]-\tfrac{1}{\beta}D_{\mathrm{KL}}(\pi\|\mu_t).
\]
The Gibbs distribution $S_{\beta,t}(f)$ maximizes this objective, and its density gives the exact identity
\[
J_{\beta,t,f}\big(S_{\beta,t}(f)\big)-J_{\beta,t,f}(\pi)
=\tfrac{1}{\beta}D_{\mathrm{KL}}\big(\pi\|S_{\beta,t}(f)\big).
\]
With $f=f_\theta$ and $\pi=S_{\beta,t}(g_t)$, this identity expresses each term in sampling regret as the learner's regularized-value gap. The factor $1/\beta$ gives this gap the units and scale of reward regret, linking distribution matching to the zero-temperature optimization limit.

\paragraph{Bandit optimization as a limit of tilted sampling.}

The Gibbs sampling operator $S_{\beta,t}$ interpolates between diffuse sampling and hard maximization. Under mild regularity conditions, as $\beta\to\infty$, the measure $S_{\beta,t}(f)$ concentrates on the maximizers of $f$ over $\mathcal A_t$. Thus UCS and TS$^2$ recover UCB and Thompson sampling in the zero-temperature limit. A precise statement is given in Appendix~\ref{app:rl-limit-sampling}.

\paragraph{Alternative sampling regrets.}

The reverse-KL regret above is our primary notion of sampling regret because it has an exact entropy-regularized value interpretation and recovers classical regret in the zero-temperature limit. Following the multi-armed sampling literature \citep{pedramfar26_multi_armed_sampl_probl_end_explor}, it is also natural to consider two other discrepancies between the learner's sampling distribution and the target Gibbs distribution.

First, one may define the total-variation sampling regret
\[
    \mathrm{TV\text{-}SRegret}(T)
    =
    \tfrac{1}{\sqrt{\beta}}
    \textstyle \sum_{t=1}^T
    d_{\rm TV}\!\left(
        S_{\beta,t}(g_t),
        S_{\beta,t}(f_\theta)
    \right).
\]
Unlike reverse-KL regret, total-variation regret does not coincide with entropy-regularized value suboptimality and does not recover the usual notion of regret as $\beta\to\infty$.

Second, one may consider the forward-KL sampling regret
\[
    \mathrm{FKL\text{-}SRegret}(T)
    =
    \tfrac{1}{\beta}
    \textstyle \sum_{t=1}^T
    D_{\rm KL}\!\left(
        S_{\beta,t}(f_\theta)
        \,\middle\|\,
        S_{\beta,t}(g_t)
    \right).
\]
Forward KL is a natural complementary measure: while reverse KL penalizes the learner for placing mass where the target places little mass, forward KL penalizes the learner for failing to cover regions where the target places mass.

Our main results focus on reverse-KL regret because of its variational interpretation. Nevertheless, our decompositions also control these alternative regret notions. In particular, SWC and UCS obtain matching reverse- and forward-KL bounds from Lemma~\ref{lem:frequentist-kl-bound}, while TS$^2$ obtains them through the same Bayesian KL-divergence argument. Total-variation regret can be related to KL regret via Pinsker's inequality,
$d_{\rm TV}(\pi,\pi^*) \le \sqrt{\tfrac{1}{2}D_{\rm KL}(\pi\|\pi^*)}$.
Specifically, we see that
\begin{align*}
  \mathrm{TV\text{-}SRegret}(T)
  &=
  \frac{1}{\sqrt{\beta}}
  \textstyle \sum_{t=1}^T
  d_{\rm TV}\!\left(
    S_{\beta,t}(g_t),
    S_{\beta,t}(f_\theta)
  \right) \\
  &\leq 
  \tfrac{1}{\sqrt{\beta}}
  \textstyle \sum_{t=1}^T
  \sqrt{\tfrac{1}{2} D_{\rm KL}\left(
    S_{\beta,t}(g_t),
    S_{\beta,t}(f_\theta)
  \right) } \\
  &\leq 
  \sqrt{
    \tfrac{T}{2 \beta} \textstyle \sum_{t=1}^T
    D_{\rm KL}\left(
    S_{\beta,t}(g_t),
    S_{\beta,t}(f_\theta)
  \right) } \\
  &=
  \sqrt{ \tfrac{T}{2} \mathrm{SRegret}(T) }.
\end{align*}
Because of the symmetry, the same bound also holds if we replace $\mathrm{SRegret}$ with $\mathrm{FKL\text{-}SRegret}$.
Thus, if either of the KL regrets is logarithmic, then the total variation regret is $\tilde{\mathcal O}(\sqrt{T})$.
In the following, we analyze reverse-KL and forward-KL sampling regrets.

\subsection{Algorithms}

As before, we define confidence sets $\mathcal F_t$.
We consider two natural analogues of UCB and TS.
In \emph{Upper Confidence Sampling (UCS)}, we choose $g_t = U_t$. In other words, we sample according to the upper confidence function
$A_t \sim S_{\beta,t}(U_t).$

On the other hand, in \emph{Thompson Sampling for Sampling (TS$^2$)}, we sample $\tilde\theta_t \sim \rho(\cdot \mid H_t,\mathcal A_t)$ and select $g_t = f_{\tilde\theta_t}$. Thus, we have
$A_t \sim S_{\beta,t}(f_{\tilde\theta_t}).$

Note that if we replace $S_{\beta}$ with $\arg\max$, then these algorithms reduce to classical UCB and TS. 
More precisely, under mild regularity assumptions, $S_\beta(f)$ concentrates on the maximizers of $f$ as $\beta\to\infty$; when the maximizer is unique, this convergence is to the corresponding Dirac measure.
Therefore, UCS and TS$^2$ can be viewed as smooth generalizations of UCB and Thompson sampling, with the Gibbs operator interpolating between sampling and maximization.

We say an algorithm follows the \emph{Sampling With Confidence (SWC)} framework, if we have
\begin{align*}
L_t(a) \leq g_t(a) \leq U_t(a),
\end{align*}
for all $a \in \mathcal{A}_t$.
Note that this does not require $g_t$ to belong to $\mathcal{F}_t$ or even $\mathcal{F}$. Rather, it controls the range of possible values of $g_t(a)$ with the uncertainty induced by the confidence sets.
Clearly UCS follows the SWC framework, since it chooses $g_t = U_t$.

\subsection{Sampling Regret Decomposition}
\label{subsec:decomp}

The regret decompositions connect KL divergence between Gibbs distributions to uncertainty in their reward functions. Global Gibbs inequalities control the discrepancy between $S_{\beta,t}(g_t)$ and $S_{\beta,t}(f_\theta)$ through function differences evaluated at sampled actions. Confidence bounds then control these differences by $w_{\mathcal F_t}$, yielding linear- and squared-width decompositions. SWC, UCS, and TS$^2$ use this common structure in different ways.

Throughout this section, the SWC and UCS bounds use a fixed unknown $\theta$, with $\mathbb E$ and $\mathbb P$ taken over observation noise and environmental or algorithmic randomness. The TS$^2$ bounds use the joint Bayesian law with $\theta\sim\rho$, including the prior draw and posterior-sampling randomness. Confidence events and width processes are evaluated under the corresponding algorithm's law.

\begin{theorem}[Sampling Regret for SWC framework]
\label{thm:regret-decom-swc}
Let $\mathcal E$ denote the event $\{ f_\theta \in \cap_{t=1}^T \mathcal F_t \}$.
If the agent follows SWC framework, we have
\begin{align*}
\mathbb E[\mathrm{SRegret}(T)]
&\le
\tfrac{1}{2} \beta \exp(\beta C) \B{E}\left[ \textstyle \sum_{t=1}^T (w_{\mathcal F_t}(A_t))^2 \right] \\
&\qquad+ 2 C T \mathbb P(\mathcal E^c).
\end{align*}
Moreover, the same bound holds with
$\mathbb E[\mathrm{SRegret}(T)]$ replaced by
$\mathbb E[\mathrm{FKL\text{-}SRegret}(T)]$.
\end{theorem}

The idea of the proof is to use Lemma~\ref{lem:frequentist-kl-bound} to bound $D_{\mathrm{KL}}$ with the squared gap between $g_t$ and $f_\theta$.
See Appendix~\ref{apx:regret-decom-swc} for the detailed proof.

\paragraph{From prediction error to sampling regret.}
The squared-width structure begins with prediction error. For an SWC procedure with $\pi_t=S_{\beta,t}(g_t)$, the confidence event $f_\theta\in\mathcal F_t$ gives
\[
\mathbb E_{A\sim\pi_t}\!\left[(g_t(A)-f_\theta(A))^2\right]
\leq
\mathbb E_{A\sim\pi_t}\!\left[w_{\mathcal F_t}(A)^2\right].
\]
This bound has unit constant and weights reward errors according to the learner's sampling distribution, giving little weight to errors in rarely sampled regions. The global Gibbs/KL inequalities in Lemmas~\ref{lem:frequentist-kl-bound} and~\ref{lem:ts-kl-bound-2} convert such prediction-error control into bounds on the discrepancy from the target distribution. They extend the local quadratic geometry of KL to finite-time guarantees throughout learning. Accumulating the squared widths determines the time dependence, while the conversion to KL determines the temperature dependence. For generic SWC, this conversion gives the exponential factor in Theorem~\ref{thm:regret-decom-swc}. Optimism and posterior averaging lead to the sharper bounds below.

\paragraph{Optimism and exploration.}
Greedy sampling based on a reward estimate $g_t$ is an SWC procedure whenever the estimate lies within the confidence envelope. At finite $\beta$, Gibbs sampling can explore beyond the current estimated maximizers. As $\beta$ increases, this exploration diminishes as the distribution concentrates on those maximizers. In the hard-maximization limit, a greedy rule can repeatedly select an apparently superior, well-estimated action while overlooking an underestimated alternative. Optimism directs exploration toward uncertain actions to avoid this failure. For UCS, it also removes the exponential factor from the squared-width bound and provides a temperature-independent linear-width bound.

\begin{theorem}[Sampling Regret of UCS]
\label{thm:regret-decom-ucs}
Let $\mathcal E$ denote the event $\{ f_\theta \in \cap_{t=1}^T \mathcal F_t \}$.
For UCS, we have
\begin{align*}
\mathbb E[\mathrm{SRegret}(T)]
&\le
\mathbb E\!\left[\textstyle \sum_{t=1}^T w_{\mathcal F_t}(A_t) \right]
+ 2 C T \mathbb P(\mathcal E^c),\\
\mathbb E[\mathrm{SRegret}(T)]
&\le
\tfrac{1}{2} \beta \mathbb E\!\left[\textstyle \sum_{t=1}^T (w_{\mathcal F_t}(A_t))^2 \right]
+ 2 C T \mathbb P(\mathcal E^c).
\end{align*}
Moreover, both inequalities remain valid with
$\mathbb E[\mathrm{SRegret}(T)]$ replaced by
$\mathbb E[\mathrm{FKL\text{-}SRegret}(T)]$.
\end{theorem}

See Appendix~\ref{apx:regret-decom-ucs} for the proof.

\begin{theorem}[Bayesian Sampling Regret of TS$^2$]
\label{thm:regret-decom-ts2}
Let $\mathcal E$ denote the event $\{ f_\theta \in \cap_{t=1}^T \mathcal F_t \}$.
For TS$^2$, we have
\begin{align*}
\mathrm{BSRegret}(T)
&\le
\mathbb E\!\left[\textstyle \sum_{t=1}^T w_{\mathcal F_t}(A_t)\right]
+ 4 C T \mathbb P(\mathcal E^c),\\
\mathrm{BSRegret}(T)
&\le
\beta \mathbb E\!\left[\textstyle \sum_{t=1}^T (w_{\mathcal F_t}(A_t))^2 \right]
+ 4 C T \mathbb P(\mathcal E^c).
\end{align*}
Moreover, both inequalities remain valid with
$\mathrm{BSRegret}(T)$ replaced by
$\mathbb E[\mathrm{FKL\text{-}BSRegret}(T)]$.
\end{theorem}

Posterior exchangeability gives TS$^2$ its two-sided KL guarantee: conditional on the history and action set, the true parameter and posterior draw are i.i.d., and restricting to the confidence event preserves this symmetry. Swapping them equates the expected directed-KL terms; Lemma~\ref{lem:ts-kl-bound-2} yields the squared-width bound, while Lemma~\ref{lem:ts-kl-bound-1} gives the linear-width bound. For UCS, pointwise optimism controls both KL directions with identical exponential-free bounds; generic SWC incurs the factor $\exp(\beta C)$. See Appendices~\ref{apx:regret-decom-ucs}, \ref{apx:regret-decom-swc}, and \ref{apx:regret-decom-ts2} for details.

\paragraph{Role of $\beta$ and geometric regimes.}
UCS and TS$^2$ admit temperature-independent linear-width bounds and temperature-dependent squared-width bounds. Write $D=(1+\gamma_T)d_E$ in Theorem~\ref{thm:main}; up to logarithmic factors, the resulting guarantee is
$\min\!\left\{\sqrt{DT},\,\beta D\right\},$
so the squared-width bound is smaller when $\beta\leq\sqrt{T/D}$. Figure~\ref{fig:regret-min-bounds} illustrates the schematic case $D\asymp\log T$, with crossover $\beta^\star(T)=\sqrt{T/\log T}$. Thus, at fixed $\beta$, the logarithmic branch is eventually smaller as $T$ grows; at fixed $T$, sufficiently large $\beta$ favors the square-root branch. The crossover depends on statistical complexity and reward normalization.

By contrast, writing $W_2(T)=\mathbb E[\sum_{t=1}^T w_{\mathcal F_t}(A_t)^2]$, generic SWC's reverse- and forward-KL bounds are sublinear when $\beta\exp(\beta C)W_2(T)=o(T)$, assuming sublinear confidence-failure terms. No temperature-independent sublinear guarantee holds uniformly over SWC, which includes greedy procedures whose hard-maximization limits can incur linear regret. This contrast highlights the additional guarantees of UCS and TS$^2$.

\section{A Black-Box Transfer of Confidence-Width Analyses to Sampling Regret}

Theorems~\ref{thm:regret-decom-ucs} and~\ref{thm:regret-decom-ts2} separate the analysis of tilted sampling into two components: a Gibbs/KL decomposition and a statistical bound on the accumulated confidence widths. This separation provides a black-box transfer principle for analyses in the UCB and Thompson-sampling framework of~\citet{russo14_learn_optim_via_poster_sampl}. Replacing maximization with Gibbs sampling gives UCS and TS$^2$; supplying bounds on their width processes gives sampling-regret guarantees for every fixed $\beta>0$. Linear-width bounds give temperature-independent guarantees, and squared-width bounds give the complementary temperature-dependent guarantees.

The underlying model has direct noisy reward feedback and no unknown transition dynamics. The statistical inputs are confidence sets valid under the sampling procedure and bounds on the widths along its action sequence. These inputs can come from standard confidence-set and Eluder-dimension arguments or from other model-specific constructions. Each construction can be combined with the same Gibbs/KL decomposition.

Self-normalized concentration and Eluder-dimension arguments used in bandit optimization often control squared widths as well as linear widths. The following subsections develop this route using the framework of~\citet{russo14_learn_optim_via_poster_sampl,russo13_eluder_dimen_sampl_compl_optim_explor}: we construct confidence sets, bound their squared-width process, and combine these bounds with the sampling-regret decompositions.

\subsection{Confidence Sets}
\label{subsec:confidence-sets}

We now discuss a general construction of confidence sets based on least-squares estimation.
The key point is that this construction, together with standard concentration results and Eluder-dimension arguments, is classical and independent of the sampling objective itself.

Assume throughout that
the observation noise is conditionally $\sigma$-sub-Gaussian.
Define the cumulative squared loss
$L_{2,t}(f) = \sum_{s<t}(f(A_s)-Y_s)^2$,
and let
$\hat f_t^{\mathrm{LS}}
\in
\arg\min_{f\in\mathcal F} L_{2,t}(f)$
denote a least-squares estimator.
For a function $h:\mathcal A\to\mathbb R$, define the empirical norm
$\|h\|_{2,E_t}^2 := \sum_{s<t} h(A_s)^2$.
Following~\cite{russo14_learn_optim_via_poster_sampl}, we consider confidence sets of the form
\[
\mathcal F_t
=
\left\{
f\in\mathcal F:
\|f-\hat f_t^{\mathrm{LS}}\|_{2,E_t}
\le
\sqrt{\gamma_t}
\right\},
\]
where the increasing sequence $\gamma_t$ is chosen so that the confidence sets contain $f_\theta$ uniformly over time with high probability.
In the setting considered in~\cite{russo14_learn_optim_via_poster_sampl}, noise is assumed to be sub-Gaussian and $\gamma_t$ depends on the sub-Gaussian noise parameter, the confidence level, and the metric entropy of $\mathcal F$.
More precisely, if $\op{dim}_K(\mathcal F)$ denotes the Kolmogorov dimension of $\mathcal F$, then Propositions~6 and~7 of~\cite{russo14_learn_optim_via_poster_sampl} show that, for a fixed horizon $T$, one may choose an increasing sequence $\gamma_t$ such that
$\max_{1\le t\le T}\gamma_t = \widetilde O\!\left(1+\op{dim}_K(\mathcal F)\right)$
and the confidence event
$\mathcal E
=
\left\{
f_\theta \in \cap_{t=1}^T \mathcal F_t
\right\}$
satisfies
$\mathbb P(\mathcal E^c)\le 1/T$.
In particular, for standard finite-dimensional parametric classes, one may choose
$\max_{1\le t\le T}\gamma_t = \widetilde O(d)$,
where $d$ is the effective dimension of the model class.
If $\mathcal F$ is finite, one may choose, for every $t\le T$, we have
$\gamma_t = 8\sigma^2 \log\!\left(2|\mathcal F|T\right)$,
which follows directly from the finite-class confidence bound in~\cite{russo14_learn_optim_via_poster_sampl}.

\begin{table*}[!t]
\centering
\small
\setlength{\tabcolsep}{4pt}
\begin{tabular}{@{}lcccc@{}}
\toprule
Problem & $\gamma_T$ & $d_E$ & $(1+\gamma_T)d_E$ & Prior special case \\
\midrule

$K$-armed bandit
& $\widetilde O(\log K)$
& $\widetilde O(K)$
& $\widetilde O(K\log K)$
& SWC--RKL/FKL \cite{pedramfar26_multi_armed_sampl_probl_end_explor} \\
\midrule

Finite class $\mathcal F$
& $\widetilde O(\log|\mathcal F|)$
& $d_E$
& $\widetilde O(d_E\log|\mathcal F|)$
& \shortstack{UCS--RKL~\cite{zhao2025logarithmic}\\Greedy--RKL--pref.~\cite{wu2025greedy}} \\
\midrule

Linear bandit ($d$-dim.)
& $\widetilde O(d)$
& $\widetilde O(d)$
& $\widetilde O(d^2)$
& --- \\
\midrule

Sparse linear ($s$-sparse)
& $\widetilde O(s\log d)$
& $\widetilde O(d)$
& $\widetilde O(s d\log d)$
& --- \\
\midrule

GLM ($d$-dim.)
& $\widetilde O(d)$
& $\widetilde O(d)$
& $\widetilde O(d^2)$
& --- \\

\bottomrule
\end{tabular}
\vspace{1mm}
\caption{Complexity instantiations for Theorem~\ref{thm:main}. Here $d_E=\operatorname{EC}^{\operatorname{sq}}(1,\mathcal F,T)$, and rates suppress logarithmic factors in $T$ and problem-dependent constants. Substituting the fourth column into Theorem~\ref{thm:main} gives two-sided UCS, TS$^2$, and SWC regret bounds. The final column records the limited prior special cases:~\citet{pedramfar26_multi_armed_sampl_probl_end_explor} treat finite-armed SWC for both KL directions, with bounded noise assumed for forward KL;~\citet{zhao2025logarithmic} establish the finite-class result only for UCS and reverse KL; and~\citet{wu2025greedy} analyze greedy sampling under preference feedback, with reverse-KL-equivalent regret in the Bradley--Terry specialization. Our SWC result instead uses direct reward feedback. For finite $\mathcal F$, cardinality controls $\gamma_T$, while $d_E$ remains a separate, model-dependent exploration complexity. The sparse-linear row uses the generic ambient-dimensional bound on $d_E$; sharper rates require additional structure. The generalized-linear row assumes the usual quantitative regularity conditions on the link function.
}
\label{tbl:examples}
\end{table*}

\subsection{Eluder Dimension and the Squared Width Process}

We now introduce the complexity measure controlling the squared width process.
We begin by recalling the standard notion of Eluder dimension introduced in~\cite{russo14_learn_optim_via_poster_sampl}.

\begin{definition}
Let $\mathcal F$ be a class of functions $f:\mathcal A \to \mathbb{R}$ and let $\varepsilon > 0$. An action $a \in \mathcal A$ is said to be \emph{$\varepsilon$-dependent} on $\{a_1,\dots,a_n\}$ if for all $f,g \in \mathcal F$,
$\sum_{i=1}^n (f(a_i) - g(a_i))^2 \le \varepsilon^2$
implies that
$|f(a) - g(a)| \le \varepsilon$.
Otherwise, $a$ is $\varepsilon$-independent.
The \emph{$\varepsilon$-eluder dimension} $\dim_E(\mathcal F,\varepsilon)$ is the length of the longest sequence $(a_1,\dots,a_d)$ such that each $a_t$ is $\varepsilon$-independent of $\{a_1,\dots,a_{t-1}\}$.
\end{definition}

Eluder dimension controls how long uncertainty can persist along a trajectory of actions. To quantify accumulated squared uncertainty, we use a squared-ratio variant of the eluder coefficient in~\citet[Definition~17.17]{zhang2023mathematical}. The historical discussion in Section~17.6 of that book traces the eluder-coefficient concept to~\citet{dann2021provably}, where it is called a decoupling coefficient. The definition below uses an untruncated sum of squared ratios, whereas~\citet[Definition~3.3]{zhao2025logarithmic} clips each squared term at one.

\begin{definition}
Let $\mathcal F$ be a class of functions $f:\mathcal A\to\mathbb R$, and fix $\epsilon>0$.
Given a history of actions $A_1,\dots,A_{t-1}$, define the local uncertainty at action $a\in\mathcal A$ by
\[
\sigma(a)
:=
\sup_{f,g\in\mathcal F}
\frac{
    |f(a)-g(a)|
}{
    \sqrt{\epsilon+\sum_{s<t}(f(A_s)-g(A_s))^2}
}.
\]
Note that the function $\sigma$ depends on $a$, $\epsilon$, $\mathcal F$, and the sequence of actions $\{A_i\}_{i=1}^{t-1}$.
When there is no ambiguity, we use $\sigma_t(a)$ to denote $\sigma\left(a; \epsilon, \mathcal F, \{A_i\}_{i=1}^{t-1}\right)$.
The squared eluder coefficient $\operatorname{EC}^{\operatorname{sq}}(\epsilon,\mathcal F,T)$ is the smallest number $d$ such that, for every action sequence $a_1,\dots,a_T$ and every sequence of function pairs $(f_t,g_t)_{t=1}^T$ with $f_t,g_t\in\mathcal F$,
\[
    \sum_{t=1}^T
    \frac{
        (f_t(a_t)-g_t(a_t))^2
    }{
        \epsilon+\sum_{s<t}(f_t(a_s)-g_t(a_s))^2
    }
    \le d.
\]
\end{definition}

The importance of this quantity comes from the following elementary observation. Once confidence sets are defined through empirical least-squares constraints, the definition itself immediately yields bounds for both the width process and the squared width process.
See Appendix~\ref{apx:width-eluder} for the proof of the following proposition.

\begin{proposition}
\label{prop:width-eluder}
For every $a\in\mathcal A$ and every $\epsilon>0$,
we have
$w_{\mathcal F_t}(a)
\le
\sqrt{\epsilon+4\gamma_t} \sigma_t(a)$.
Thus
\begin{align*}
\sum_{t=1}^T w_{\mathcal F_t}^2(A_t)
&\le
(\epsilon+4\gamma_T) 
\operatorname{EC}^{\operatorname{sq}}(\epsilon,\mathcal F,T),\\
\sum_{t=1}^T w_{\mathcal F_t}(A_t)
&\le
\sqrt{(\epsilon+4\gamma_T) 
\operatorname{EC}^{\operatorname{sq}}(\epsilon,\mathcal F,T) T}.
\end{align*}
\end{proposition}

Thus the same second-order quantity simultaneously controls both the sampling regret bounds of Section~\ref{sec:tilted-sampling} and the classical width process appearing in RL regret analyses.
The next result shows that the squared Eluder coefficient remains closely connected to the classical Eluder dimension.

\begin{proposition}
\label{prop:ecsq-vs-eluder}
Assume that for some constants $a,b\ge 0$ and a non-negative function $c$, we have
$\dim_E(\mathcal F,\varepsilon) \leq c(T) \varepsilon^{-a} \log(1/\varepsilon)^b$, 
for all $\varepsilon > 0$.
Then
\[
    \operatorname{EC}^{\operatorname{sq}}(\epsilon,\mathcal F,T)
    =
    \tilde{O}\left(
    \left(
    \epsilon^{-1}
    +
    c(T)(1+\epsilon^{-1})
    \right)
    T^{a/(a+2)}
    \right).
\]
\end{proposition}

\subsection{General Regret Guarantees}

We now combine the regret decompositions of Section~\ref{sec:tilted-sampling} with the confidence-set and Eluder-type analysis developed above. 
Specifically, using Theorems~\ref{thm:regret-decom-swc},~\ref{thm:regret-decom-ucs},~\ref{thm:regret-decom-ts2}, Proposition~\ref{prop:width-eluder} and the discussion in Subsection~\ref{subsec:confidence-sets}, we immediately obtain the following theorem.

\begin{theorem}
\label{thm:main}
For UCS and SWC, use expected sampling regret with fixed unknown $\theta$; for TS$^2$, use Bayesian sampling regret with $\theta\sim\rho$. Expectations and confidence-event probabilities follow the corresponding algorithm's law, as in Section~\ref{subsec:decomp}.
If confidence sets are chosen as discussed in Subsection~\ref{subsec:confidence-sets} with $\mathbb P(\mathcal E^c) = O(1/T)$, then reverse- and forward-KL regret of UCS and TS$^2$ are bounded by
\begin{align*}
\tilde{O}\left(\min\left\{ 
    \sqrt{(1+\gamma_T) d_E T}, \beta (1+\gamma_T) d_E
\right\}\right),
\end{align*}
where $d_E := \operatorname{EC}^{\operatorname{sq}}(1,\mathcal F,T)$.
Moreover, for any algorithm that follows SWC framework with respect to such a choice of confidence sets, both reverse- and forward-KL regrets are bounded by
\[\tilde{O}(\beta \exp(\beta C) (1+\gamma_T) d_E).\]
In particular, the choice $\gamma_T=\tilde O(1+\op{dim}_K(\mathcal F))$ recovers the corresponding bounds in terms of Kolmogorov dimension.
\end{theorem}

Theorem~\ref{thm:main} instantiates the sampling-regret decompositions with confidence-radius and squared-width bounds. This gives a common route from model-specific complexity bounds to guarantees for UCS, TS$^2$, and SWC.
The squared-width UCS bound recovers the optimistic decomposition of~\citet{zhao2025logarithmic} in their bandit setting when paired with their confidence sets and width bound. Finite-armed SWC procedures connect to the sampling framework of~\citet{pedramfar26_multi_armed_sampl_probl_end_explor}. The greedy analysis of~\citet{wu2025greedy} uses pairwise preference feedback; adapting its proof strategy to direct reward feedback gives the corresponding SWC-type reverse-KL guarantee. Table~\ref{tbl:examples} lists model-specific instantiations, and Appendix~\ref{app:related-work} details these connections.

\section{Limitations and Future Directions}
\label{sec:limitations}
The logarithmic-in-$T$ guarantees hold at fixed $\beta$. UCS and TS$^2$ also admit temperature-independent linear-width bounds, while generic SWC incurs the factor $\exp(\beta C)$ for both KL directions. A matching lower bound for the joint dependence on $T$, $\beta$, and function-class complexity remains open. The analysis assumes bounded rewards, direct noisy reward feedback, and exact Gibbs sampling. Accounting for approximate-sampling error and computational cost, together with empirical evaluation, are natural directions for future work.

\subsubsection*{Acknowledgements}
This research is in part supported by CIFAR AI Chairs and the NSERC Discovery program.

\bibliographystyle{apalike}
\bibliography{references}

\clearpage
\onecolumn
\appendix

\section{Related Works}
\label{app:related-work}

Our work builds on the multi-armed sampling framework of \cite{pedramfar26_multi_armed_sampl_probl_end_explor}, extending it beyond the finite-armed setting to general action spaces and learning models.

\paragraph{Entropy-regularized reinforcement learning and maximum-entropy control.}
Entropy-regularized reinforcement learning provides a principled framework for trading off reward maximization with stochasticity, robustness, and exploration. Early formulations appear in maximum entropy inverse reinforcement learning \citep{ziebart2008maximum}, where optimal policies take a Gibbs-Boltzmann form. Closely related ideas arise in linearly-solvable Markov decision processes and KL control \citep{todorov2007linearly}, as well as path-integral control formulations \citep{kappen2005path, theodorou2010generalized}, which derive stochastic optimal policies through exponential transformations of value functions.

These viewpoints were later unified under the framework of reinforcement learning as probabilistic inference \citep{levine2018reinforcement}, building on earlier connections between stochastic control and approximate inference \citep{rawlik2013stochastic}.
This connection also underlies mirror-descent and policy optimization methods with entropy regularization \citep{neu2017unified,schulman2017proximal}. It has further been shown that such entropy-regularized formulations can lead to strong optimization guarantees, including fast convergence of natural policy gradient methods under suitable conditions \citep{cen2021fast}. Building on this foundation, a sequence of modern deep RL algorithms explicitly optimize entropy-regularized objectives, including soft Q-learning \citep{haarnoja2017reinforcement} and Soft Actor-Critic \citep{haarnoja18_soft_actor_critic}, which solve problems of the form
\[
    \max_\pi \;
    \mathbb{E}_{a\sim\pi}[f_\theta(a)]
    - \frac{1}{\beta} D_{\rm KL}(\pi \| \mu).
\]
In all these formulations, the optimal policy is given by a Gibbs distribution, establishing a direct link between entropy-regularized control and Boltzmann sampling.

Despite the central role of Gibbs policies in this literature, most prior work assumes that the reward or value function is known (or can be optimized via function approximation) and focuses on computing or approximating the optimal policy. In contrast, tilted sampling studies the \emph{statistical} problem of learning the Gibbs-optimal distribution from bandit feedback when the reward function is initially unknown. This shift—from solving a regularized control problem to learning a distributional target—leads to a different regret notion and a distinct uncertainty structure governed by second-order effects.

\paragraph{RLHF, preference-based learning, and fine-tuning.}
Reinforcement learning from human feedback (RLHF) and preference-based learning provide a practical instantiation of KL-regularized objectives for aligning generative models. Early work introduced learning from pairwise preferences using reward models and reinforcement learning \citep{christiano2017deep}, while large-scale systems such as InstructGPT combine preference modeling with KL-constrained policy optimization \citep{ouyang2022training}, typically implemented using proximal policy optimization \citep{schulman2017proximal}. These approaches optimize KL-regularized RL objectives leading to Gibbs-form solutions.

More recent methods such as direct preference optimization (DPO) \citep{rafailov2023direct} and ranking-based approaches \citep{yuan2023rrhf} optimize similar objectives without explicit reinforcement learning. Variants and extensions further explore alternative regularization schemes and divergences, including formulations beyond reverse KL \citep{azar2023dpo, wang24_beyon_rever_kl}.

These objectives implicitly define a target sampling distribution, but the statistical problem of learning this distribution under feedback has received comparatively little attention, with only recent work beginning to study RLHF from an online learning perspective \citep{zhao2025logarithmic}.

\paragraph{Theoretical RLHF and preference-based online learning.}
A growing body of work studies preference-based learning and RLHF from a theoretical perspective. Early work on dueling bandits and preference learning analyzes regret under pairwise feedback \citep{yue12_k, zoghi14_relat_upper_confid_bound_k, dudik15_contex_duelin_bandit}. More recent work considers RLHF-style settings with KL-regularized objectives and online feedback, establishing regret guarantees for optimistic and greedy methods \citep{xiong24_iterat_prefer_learn_human_feedb,zhao25_sharp_analy_kl_regul_contex_bandit_rlhf,zhao2025logarithmic, wu2025greedy}. These works highlight that KL-regularized objectives can admit improved regret behavior compared to classical bandits.

Our work is complementary to this line of research. Rather than focusing on specific feedback models or algorithmic variants, we study the general problem of learning to sample from Gibbs distributions under bandit feedback. This perspective isolates the role of the sampling objective itself, leading to a unified regret framework governed by second-order uncertainty.

\paragraph{Connections to MCMC and adaptive sampling.}
Sampling from a Gibbs or Boltzmann distribution has long been studied in
computational statistics, primarily through Markov Chain Monte Carlo
(MCMC) — including Metropolis--Hastings, Langevin, and Hamiltonian Monte
Carlo
\citep{metropolis1953equation, hastings1970monte, neal2011mcmc, roberts1996exponential}.
Adaptive MCMC \citep{andrieu2008tutorial, roberts2009examples} and
Sequential Monte Carlo samplers \citep{delmoral2006smc} extend these
schemes by tuning proposal distributions or particle systems from past
samples.
In all of these settings, however, the target density — equivalently,
the reward $f_\theta$ — is assumed \emph{known and queryable pointwise};
the difficulty is purely \emph{computational}, and the associated
complexity measures are mixing-time quantities such as spectral gaps and
log-Sobolev constants
\citep{bakry2014analysis, vempala2019rapid}.
Tilted sampling instead targets the \emph{statistical} problem
of approximating a Gibbs distribution whose generating reward is
unknown and accessed only through noisy bandit feedback, with regret
controlled by a cumulative confidence-width process rather than a
mixing time.

\paragraph{GFlowNets and amortized samplers.}
GFlowNets \citep{bengio2021flow, bengio2023gflownet} learn an amortized
policy that generates compositional objects with probability
proportional to a given reward $R(x)$, with a growing line of work on
training objectives — flow matching, detailed, trajectory, and
sub-trajectory balance \citep{malkin2022trajectory, madan2023learning}
— and continuous, energy-based, or diffusion variants
\citep{lahlou2023theory, zhang2023diffusion, sendera2024diffusion}.
Because GFlowNets are trained off-policy, exploration is itself an
active subject: tempered or augmented forward policies and prioritized
replay \citep{pan2023generative, shen2023towards}, local-search
refinement of trajectories \citep{kim2024local}, adaptive teacher
proposals \citep{kim2025adaptive}, Thompson-sampling-style heuristics
for GFlowNet rollouts \citep{rectorbrooks2023thompson}, and
GFlowNet-driven active learning with epistemic uncertainty estimates
\citep{jain2022biological}.
These works treat exploration empirically and almost always assume
that $R$ is queryable; analyses focus on training-loss consistency
rather than regret under bandit feedback.
Tilted sampling complements this literature by providing a
regret-theoretic counterpart: UCS and TS$^2$ realize optimism and
posterior-sampling exploration directly in the Gibbs-sampling regime,
and our confidence-width decompositions quantify the sampling regret
of these exploration strategies across general reward-function classes.

\paragraph{Neural samplers for unnormalized targets.}
A recent body of work develops neural samplers for unnormalized
densities, including diffusion-based samplers such as Path Integral
Samplers \citep{zhang2022path}, Denoising Diffusion Samplers
\citep{vargas2023denoising}, and iterated denoising-energy matching
\citep{akhound2024idem}; controlled-diffusion and stochastic-optimal
control approaches \citep{vargas2024transport, berner2024optimal};
and normalizing-flow Boltzmann generators \citep{noe2019boltzmann}.
Mode coverage is a central practical concern in this literature, and
explicit exploration mechanisms have been proposed to address it:
Flow-Annealed Importance Sampling Bootstrap \citep{midgley2023flow}
couples a flow proposal with annealed importance sampling to escape
mode collapse, and diffusion-based samplers such as iDEM
\citep{akhound2024idem} rely on replay buffers and score-based
exploration noise to cover multimodal targets.
As with GFlowNets, however, these methods assume access to the
unnormalized target $\exp(\beta f_\theta)\,\mu$ during training; closer
to bandit feedback are RL-based fine-tuning methods for diffusion and
language models such as DDPO \citep{black2024training} and DPOK
\citep{fan2023dpok}, which optimize a \emph{learned} reward via policy
gradients but without regret guarantees.
Our framework is complementary: it formalizes the online-learning
problem underlying these procedures and shows that, under bandit
feedback and a KL-regularized objective, the regret for matching a
Gibbs target admits a second-order upper bound in the
confidence widths, with $\beta$ smoothly interpolating between
distribution matching and hard optimization.

\section{Proof of Theorem~\ref{thm:ucb-ts}}
\label{apx:ucb-ts}

\begin{proof}
\noindent \textbf{Regret decomposition for UCB:}
Fix $t \geq 1$. 
Add and subtract $U_t(A_t^\star)$:
\begin{align*}
f_{\theta}(A_t^\star)-f_{\theta}(A_t)
&=
\big(f_{\theta}(A_t^\star)-U_t(A_t^\star)\big)
+
\big(U_t(A_t^\star)-U_t(A_t)\big)
+
\big(U_t(A_t)-f_{\theta}(A_t)\big) \\
&\le
\big(f_{\theta}(A_t^\star)-U_t(A_t^\star)\big)
+
\big(U_t(A_t)-f_{\theta}(A_t)\big),
\end{align*}
because $U_t(A_t)\ge U_t(A_t^\star)$.
On the event $f_\theta \in \mathcal F_t$, we have
\begin{align*}
f_{\theta}(A_t^\star) \leq U_t(A_t^\star)
\quad\text{ and }\quad
f_{\theta}(A_t) \geq L_t(A_t),
\end{align*}
so
\begin{align*}
f_{\theta}(A_t^\star)-f_{\theta}(A_t)
\le
U_t(A_t)-L_t(A_t)
=
w_{\mathcal F_t}(A_t).
\end{align*}
If the confidence condition fails, the regret is at most $C$. Summing over $t$ and taking expectations of the failure indicator yields the result.

\noindent \textbf{Regret decomposition for TS:}
For each $t$, define
\begin{align*}
A_t^\star \in \arg\max_{a\in\mathcal A_t} f_{\theta}(a).
\end{align*}

Condition on $(H_t,\mathcal A_t)$. Since $\tilde\theta_t \sim \rho(\cdot \mid H_t,\mathcal A_t)$ and $\theta \mid (H_t,\mathcal A_t) \sim \rho(\cdot \mid H_t,\mathcal A_t)$, the random variables $\tilde\theta_t$ and $\theta$ have the same conditional distribution given $(H_t,\mathcal A_t)$. Therefore, because $A_t$ is a measurable function of $\tilde\theta_t$ and $A_t^\star$ is a measurable function of $\theta$, we have the conditional identity
\begin{align*}
\mathbb E\!\left[ U_t(A_t) \mid H_t,\mathcal A_t \right]
=
\mathbb E\!\left[ U_t(A_t^\star) \mid H_t,\mathcal A_t \right].
\end{align*}
Using this, we obtain
\begin{align*}
\mathbb E\!\left[f_\theta(A_t^\star)-f_\theta(A_t)\mid H_t,\mathcal A_t\right]
=
\mathbb E\!\left[U_t(A_t)-f_\theta(A_t)\mid H_t,\mathcal A_t\right]
+
\mathbb E\!\left[f_\theta(A_t^\star)-U_t(A_t^\star)\mid H_t,\mathcal A_t\right].
\end{align*}
Summing over $t$ and taking expectations yields
\begin{align*}
\mathrm{BayesRegret}(T)
&=
\mathbb E\!\left[\textstyle \sum_{t=1}^T \big(U_t(A_t)-f_\theta(A_t)\big)\right]
+
\mathbb E\!\left[\textstyle \sum_{t=1}^T \big(f_\theta(A_t^\star)-U_t(A_t^\star)\big)\right].
\end{align*}

Let $\mathcal E_t:=\{f_\theta\in\mathcal F_t\}$. On $\mathcal E_t$, the second term is nonpositive since
\begin{align*}
f_\theta(A_t^\star) \le U_t(A_t^\star),
\end{align*}
and the first term satisfies
\begin{align*}
U_t(A_t)-f_\theta(A_t)
\le
U_t(A_t)-L_t(A_t)
=
w_{\mathcal F_t}(A_t).
\end{align*}
On $\mathcal E_t^c$, each of
$U_t(A_t)-f_\theta(A_t)$ and
$f_\theta(A_t^\star)-U_t(A_t^\star)$ is at most $C$.
Therefore,
\begin{align*}
\mathrm{BayesRegret}(T)
&\le
\mathbb E\!\left[\textstyle \sum_{t=1}^T w_{\mathcal F_t}(A_t)\right]
+
2C\sum_{t=1}^T\mathbb P(\mathcal E_t^c) \\
&\le
\mathbb E\!\left[\textstyle \sum_{t=1}^T w_{\mathcal F_t}(A_t)\right]
+
2CT\mathbb P(\mathcal E^c).
\qedhere
\end{align*}
\end{proof}

\section{Bandit Optimization as a Limit of Tilted Sampling}
\label{app:rl-limit-sampling}

The following result is classical and may be seen as an application of Varadhan's lemma. (See, e.g., \cite{dembo09_large_deviat_techn_applic,ellis06_entrop_large_deviat_statis_mechan})
Here we include a proof for completeness.

\begin{lemma}\label{lem:boltzmann-limit-max}
Assume $\mathcal{A}$ is a metric space with a Borel probability measure $\mu$.
Let $f: \mathcal{A} \to \mathbb{R}$ be bounded and continuous and define
\[
M := \sup_{a \in \mathcal{A}} f(a), \qquad \mathcal{A}^* := \op{argmax}_{a \in \mathcal{A}} f(a) = \{a \in \mathcal{A} : f(a)=M\}.
\]
Then the following are true.
\begin{enumerate}
\item[(i)] If $\mathcal{A}^*$ is nonempty and
$\mathcal{A}^* \subseteq \op{supp}(\mu)$, then
\[
\lim_{\beta \to \infty} \mathbb{E}_{a \sim S_\beta(f)}[f(a)] = M.
\]
Moreover, any weak limit of $S_\beta(f)$, as $\beta \to \infty$, has zero measure on $\mathcal{A} \setminus \mathcal{A}^*$.
\item[(ii)] If $\mu(\mathcal{A}^*)>0$, then the measures $S_\beta(f)$ converge setwise (and hence weakly) as $\beta \to \infty$ to the probability measure $\mu^*$ given by
\[
\mu^*(E) := \frac{\mu(E \cap \mathcal{A}^*)}{\mu(\mathcal{A}^*)}.
\]
In particular, $\mu^*(\mathcal{A} \setminus \mathcal{A}^*) = 0$ and
$\lim_{\beta \to \infty} \mathbb{E}_{a \sim S_\beta(f)}[f(a)] = M$.
\end{enumerate}
\end{lemma}

Note that if $\mathcal{A}^*=\{a^*\}$ is a singleton, then the only probability measure $\nu$ such that $\nu(\mathcal{A} \setminus \mathcal{A}^*) = 0$ is the Dirac measure at $a^*$.

\begin{remark}[Zero-measure maximizers]
If $\mu(\mathcal A^*)=0$, then setwise convergence in part (ii) cannot hold, since each $S_\beta(f)$ is absolutely continuous with respect to $\mu$ and hence assigns zero mass to $\mathcal A^*$ for every $\beta<\infty$, whereas any weak limit is supported on $\mathcal A^*$ by part (i). Thus one can only expect weak convergence, and even this may fail to be unique without further assumptions. \\
In general, $S_\beta(f)$ concentrates on $\mathcal A^*$ as $\beta\to\infty$, but the limiting distribution (if it exists) depends on the local behavior of $f$ and $\mu$ near $\mathcal A^*$. For example, if $\mathcal A\subset\mathbb R^d$, $\mu$ has a continuous density, $f\in C^2$, and $\mathcal A^*=\{a_1,\dots,a_k\}$ consists of finitely many nondegenerate maximizers, then Laplace’s method yields weak convergence to a discrete measure supported on $\mathcal A^*$ with weights determined by second-order asymptotics of $f$ and $\mu$ at the maximizers. More generally, analogous descriptions hold when $\mathcal A^*$ is a smooth manifold under suitable nondegeneracy conditions. \\
These phenomena are classical and can be understood via Laplace asymptotics or, more abstractly, through the Laplace principle in large deviations theory; see, e.g., \cite{dembo09_large_deviat_techn_applic,ellis06_entrop_large_deviat_statis_mechan}.
\end{remark}

\begin{proof}
We first prove an auxiliary claim.
For $\varepsilon>0$, define
\[
A_\varepsilon := \{a\in\mathcal{A} : f(a) \le M-\varepsilon\}.
\]
If $\mu(A_{\varepsilon/2}^c) > 0$, then we have
\[
\lim_{\beta\to\infty} S_\beta(f)(A_\varepsilon) = 0.
\]
 
\noindent
\emph{Proof of the claim.}

Define the function $g_\varepsilon(a) := f(a) - M + \varepsilon$.
Then for all $a\in\mathcal{A}$, we have
$e^{\beta f(a)} = e^{\beta(M-\varepsilon)} e^{\beta g_\varepsilon(a)}$, 
and the factor $e^{\beta(M-\varepsilon)}$ cancels in the ratio defining $S_\beta(f)(A_\varepsilon)$, yielding
\[
S_\beta(f)(A_\varepsilon)
= \frac{\int_{A_\varepsilon} e^{\beta g_\varepsilon(a)}\,d\mu(a)}
       {\int_{\mathcal{A}} e^{\beta g_\varepsilon(a)}\,d\mu(a)}.
\]

On $A_\varepsilon$, we have $f(a)\le M-\varepsilon$, hence $g_\varepsilon(a)\le 0$ and so $e^{\beta g_\varepsilon(a)} \le 1$ for all $a\in A_\varepsilon$, $\beta>0$.
Therefore,
\[
\int_{A_\varepsilon} e^{\beta g_\varepsilon(a)}\,d\mu(a) \le \mu(A_\varepsilon).
\]

On the other hand, we have
\[
A_{\varepsilon/2}^c = \mathcal{A}\setminus A_{\varepsilon/2} = \{a\in\mathcal{A}: f(a) > M-\varepsilon/2\} \subseteq \{a\in\mathcal{A}: f(a) > M-\varepsilon\} = A_{\varepsilon}^c,
\]
and over the set $A_{\varepsilon/2}^c$, we have $g_\varepsilon(a) = f(a) - M + \varepsilon> \varepsilon/2$.
Hence for all $a \in A_{\varepsilon/2}^c$ and all $\beta>0$, we have
$e^{\beta g_\varepsilon(a)} \ge e^{\beta \varepsilon/2}$,
so
\[
\int_{A_{\varepsilon/2}^c} e^{\beta g_\varepsilon(a)}\,d\mu(a)
\ge e^{\beta \varepsilon/2}\mu(A_{\varepsilon/2}^c).
\]
Therefore
\[
\int_{\mathcal{A}} e^{\beta g_\varepsilon(a)}\,d\mu(a)
\ge \int_{A_{\varepsilon/2}^c} e^{\beta g_\varepsilon(a)}\,d\mu(a)
\ge e^{\beta \varepsilon/2} \mu(A_{\varepsilon/2}^c).
\]
Combining the bounds on numerator and denominator, we obtain
\[
S_\beta(f)(A_\varepsilon)
\le \frac{\mu(A_\varepsilon)}{e^{\beta \varepsilon/2} \mu(A_{\varepsilon/2}^c)}
= \frac{\mu(A_\varepsilon)}{\mu(A_{\varepsilon/2}^c)}\, e^{-\beta \varepsilon/2}
\xrightarrow[\beta\to\infty]{} 0.
\]

\noindent
\emph{Proof of (i).}
The function $f$ is continuous, the set $A_\varepsilon^c \supseteq \mathcal{A}^*$ is open.
Since $\mathcal{A}^* \subseteq \op{supp}(\mu)$, by definition we have $\mu(U) > 0$ for any open sets $U$ that intersects $\mathcal{A}^*$.
In particular, we have $\mu(A_{\varepsilon/2}^c) > 0$.
Hence we may use the above claim to see that
\[
\lim_{\beta\to\infty} S_\beta(f)(A_\varepsilon) = 0,
\]
for all $\varepsilon > 0$.

First we prove the claim regarding the expectation.
Let $m:=\inf_{a\in\mathcal A}f(a)$, which is finite since $f$ is bounded.
For every $\varepsilon>0$, we have
\[
\mathbb{E}_{a\sim S_\beta(f)}[f(a)]
\geq
(M-\varepsilon)S_\beta(f)(A_\varepsilon^c)
+
mS_\beta(f)(A_\varepsilon).
\]
Since $S_\beta(f)(A_\varepsilon)\to0$, it follows that
\[
\liminf_{\beta\to\infty}
\mathbb{E}_{a\sim S_\beta(f)}[f(a)]
\geq M-\varepsilon.
\]
Letting $\varepsilon\downarrow0$ and using $f(a)\leq M$ gives
\[
  \lim_{\beta \to \infty} \mathbb{E}_{a \sim S_\beta(f)}[f(a)] = M.
\]

Now suppose $(\beta_n)$ is any sequence with $\beta_n\to\infty$, and suppose $S_{\beta_n}(f)$ converges weakly to some probability measure $\nu$ on $\mathcal{A}$. 
Since the set $B_\varepsilon := \{a\in\mathcal{A} : f(a) < M-\varepsilon\} \subseteq A_\varepsilon$ is open, by the Portmanteau theorem we have
\[
\nu(B_\varepsilon) 
\leq \limsup_{n\to\infty} S_{\beta_n}(f)(B_\varepsilon)
\leq \limsup_{n\to\infty} S_{\beta_n}(f)(A_\varepsilon)
= 0.
\]
Thus
\begin{align*}
\nu(\mathcal{A}\setminus \mathcal{A}^*)
&= \nu \left( \bigcup_{k=1}^\infty \{a\in\mathcal{A} : f(a) < M-1/k\} \right) 
= \nu \left( \bigcup_{k=1}^\infty B_{1/k} \right) 
\leq \sum_{k=1}^\infty \nu(B_{1/k})
= 0.
\end{align*}
That is, any weak limit $\nu$ of $S_\beta(f)$ is supported on $\mathcal{A}^*$.

\medskip

\noindent
\emph{Proof of (ii).}
Similar to before, we write $f(a) = M + (f(a)-M)$ and factor out $e^{\beta M}$ to see that
\[
S_\beta(f)(E)
= \frac{\int_E e^{\beta(f(a)-M)}\,d\mu(a)}
       {\int_{\mathcal{A}} e^{\beta(f(b)-M)}\,d\mu(b)}.
\]
Define
\[
D_\beta 
:= \int_{\mathcal{A}} e^{\beta(f(b)-M)}\,d\mu(b)
= \int_{\mathcal{A}^*} e^{\beta(f(b)-M)}\,d\mu(b)
        + \int_{\mathcal{A}\setminus\mathcal{A}^*} e^{\beta(f(b)-M)}\,d\mu(b).
\]
On $\mathcal{A}^*$ we have $f(b)=M$, so $f(b)-M=0$ and $e^{\beta(f(b)-M)}=1$. Thus
\[
\int_{\mathcal{A}^*} e^{\beta(f(b)-M)}\,d\mu(b) = \mu(\mathcal{A}^*).
\]
For $\varepsilon>0$ and define
\[
C_\varepsilon := (\mathcal{A}\setminus\mathcal{A}^*)\setminus A_\varepsilon
               = \{b\in\mathcal{A}\setminus\mathcal{A}^* : f(b)>M-\varepsilon\}.
\]
Then \(\mathcal{A}\setminus\mathcal{A}^* = A_\varepsilon \cup C_\varepsilon\) and \(A_\varepsilon \cap C_\varepsilon = \emptyset\).
On $A_\varepsilon$, we have $f(b)-M\le -\varepsilon$, so
\[
0 \le e^{\beta(f(b)-M)} \le e^{-\beta\varepsilon},
\]
and hence
\[
0 \le \int_{A_\varepsilon} e^{\beta(f(b)-M)}\,d\mu(b)
   \le e^{-\beta\varepsilon}\mu(A_\varepsilon)
   \xrightarrow[\beta\to\infty]{} 0.
\]

On $C_\varepsilon$, we have $f(b)-M<0$, hence $0<e^{\beta(f(b)-M)}<1$, so
\[
0 \le \int_{C_\varepsilon} e^{\beta(f(b)-M)}\,d\mu(b)
   \le \mu(C_\varepsilon).
\]
As $\varepsilon\downarrow 0$, the sets $C_\varepsilon$ form a decreasing sequence with $\cap_{\varepsilon>0} C_\varepsilon = \emptyset$.
Since $\mu$ is finite, continuity from above implies that
$\mu(C_\varepsilon) \to 0$ as $\varepsilon \to 0$.
Therefore
\begin{align*}
\limsup_{\beta\to\infty} \int_{\mathcal{A}\setminus\mathcal{A}^*} e^{\beta(f(b)-M)}\,d\mu(b)
&\le \limsup_{\varepsilon \to 0} \left(\limsup_{\beta\to\infty} \int_{A_\varepsilon} e^{\beta(f(b)-M)}\,d\mu(b) + \mu(C_\varepsilon) \right) \\
&= \limsup_{\varepsilon \to 0} \mu(C_\varepsilon)
= 0.
\end{align*}
Thus
\[
\lim_{\beta\to\infty} D_\beta 
= \mu(\mathcal{A}^*) + \lim_{\beta\to\infty} \int_{\mathcal{A}\setminus\mathcal{A}^*} e^{\beta(f(b)-M)}\,d\mu(b) 
= \mu(\mathcal{A}^*).
\]

Now fix any measurable $E\subseteq\mathcal{A}$. We have
\[
\int_E e^{\beta(f(a)-M)}\,d\mu(a)
= \int_{E\cap \mathcal{A}^*} e^{\beta(f(a)-M)}\,d\mu(a)
 + \int_{E\setminus \mathcal{A}^*} e^{\beta(f(a)-M)}\,d\mu(a).
\]
On $E\cap\mathcal{A}^*$, $f(a)=M$, so the first integral is $\mu(E\cap\mathcal{A}^*)$. On $E\setminus\mathcal{A}^*$ we have $f(a)<M$, and by the same argument as above, the second integral tends to $0$ as $\beta\to\infty$. Hence
\[
\lim_{\beta\to\infty} \int_E e^{\beta(f(a)-M)}\,d\mu(a)
= \mu(E\cap\mathcal{A}^*).
\]

Combining this with $\lim_{\beta\to\infty} D_\beta = \mu(\mathcal{A}^*)$, we obtain
\[
\lim_{\beta\to\infty} S_\beta(f)(E)
= \lim_{\beta\to\infty} \frac{\int_E e^{\beta(f(a)-M)}\,d\mu(a)}{D_\beta}
= \frac{\mu(E\cap\mathcal{A}^*)}{\mu(\mathcal{A}^*)}
=: \mu^*(E).
\]
This is convergence for each measurable $E$, i.e.\ setwise convergence. 
In particular, it implies weak convergence of the measures $S_\beta(f)$ to $\mu^*$ as $\beta\to\infty$.
Using Portmanteau theorem, we have
\begin{align*}
\lim_{\beta \to \infty} \mathbb{E}_{a \sim S_\beta(f)}[f(a)] 
&= \mathbb{E}_{a \sim \mu^*}[f(a)] 
= M.
\qedhere
\end{align*} 
\end{proof}

\section{Frequentist bounds on KL-divergence}

\begin{lemma}\label{lem:jensen-gap}
Let $(\C{X}, \nu)$ be a measure space and let $f : \mathcal{X} \to \mathbb{R}$ be a random variable.
\begin{enumerate}
\item[(i)] If $f$ is non-negative, then
\begin{align*}
\B{E}[f] + \log(\B{E}[\exp(-f)]) \leq \B{E}[f].
\end{align*}
\item[(ii)] If $f \geq -M$ for some $M \geq 0$, then
\begin{align*}
\B{E}[f] + \log(\B{E}[\exp(-f)]) \leq \tfrac{1}{2}\exp(M)\B{E}[f^2].
\end{align*}
\end{enumerate}
\end{lemma}

\begin{proof}
\textbf{(i):}
The claim follows from the fact that $f \geq 0$ and therefore
\begin{align*}
\log(\B{E}[\exp(-f)])
\leq \log(\B{E}[\exp(0)]) = 0.
\end{align*}

\textbf{(ii):}
For $x \in \B{R}$, Taylor's theorem with the Lagrange remainder gives
\[
\exp(-x) \;=\; 1 - x + \tfrac{1}{2} \exp(-\xi_1)x^2,
\quad
\exp(-x) \;=\; 1 - x + \tfrac{x^2}{2} - \tfrac{1}{6} \exp(-\xi_2)x^3,
\]
For some $\xi_1, \xi_2$ between $0$ and $x$.
Hence, for all $x \in [-M , 0]$, we have $0 \leq -\xi_1 \leq M$ and
\[
\exp(-x) 
\leq 1 - x + \frac{\exp(-\xi_1)x^2}{2}
\leq 1 - x + \frac{\exp(M)x^2}{2}.
\]
On the other hand, for all $x \geq 0$, we have $-x^3 \leq 0$ and
\[
\exp(-x) \;=\; 1 - x + \frac{x^2}{2} - \frac{\exp(-\xi_2)x^3}{6}
\leq 1 - x + \frac{x^2}{2}.
\]
Hence the point-wise bound $\exp(-x) \le 1 - x + \tfrac{1}{2}\exp(M)x^2$ holds for all $x \geq -M$. Applying this with $x=f(s)$ for $s \in \C{X}$ and taking expectations yields
\[
\mathbb{E}[\exp(-f)] \;\le\; 1 - \mathbb{E}[f] + \tfrac{1}{2}\exp(M)\,\mathbb{E}[f^2].
\]
Since logarithm is increasing and $\log(u) \leq u - 1$ for all $u>0$, we obtain
\[
\log(\mathbb{E}[\exp(-f)])
\;\le\;
\mathbb{E}[\exp(-f)] - 1
\;\le\;
-\,\mathbb{E}[f] + \tfrac12\,\exp(M) \mathbb{E}[f^2],
\]
which completes the proof.
\end{proof}

\begin{remark}
The constant $1/2$ is optimal among bounds of the form $\exp(-x) \le 1 - x + c x^2$ valid for all $x \ge 0$, since
\[
\sup_{x>0}\frac{\exp(-x) - 1 + x}{x^2} = \frac{1}{2}.
\]
\end{remark}

\begin{lemma}\label{lem:reverse-jensen-gap}
Let $(\C{X},\nu)$ be a probability space and let $f:\C{X}\to\B{R}$ be a measurable function.
\begin{enumerate}
\item[(i)] If $f$ is nonnegative, then
\begin{align*}
-\frac{\B{E}[f\exp(-f)]}{\B{E}[\exp(-f)]}
-\log\B{E}[\exp(-f)]
\leq \B{E}[f].
\end{align*}
\item[(ii)] If $f\geq-M$ for some $M\geq0$, then
\begin{align*}
-\frac{\B{E}[f\exp(-f)]}{\B{E}[\exp(-f)]}
-\log\B{E}[\exp(-f)]
\leq \tfrac{1}{2}\exp(M)\B{E}[f^2].
\end{align*}
\end{enumerate}
\end{lemma}

\begin{proof}
\textbf{(i):}
If $\B{E}[f]=\infty$, the claim is immediate. Otherwise, since $f\geq0$,
we have
\[
-\frac{\B{E}[f\exp(-f)]}{\B{E}[\exp(-f)]}
\leq 0.
\]
Moreover, Jensen's inequality gives
\[
\log\B{E}[\exp(-f)]
\geq
\B{E}[\log\exp(-f)]
=-\B{E}[f].
\]
Combining these inequalities proves the first claim.

\textbf{(ii):}
We first prove the claim when $f$ is bounded. Define
\[
\phi(s):=\log\B{E}[\exp(-sf)],
\qquad 0\leq s\leq1,
\]
and let $\nu_s$ be the probability measure with density
\[
\frac{d\nu_s}{d\nu}
=
\frac{\exp(-sf)}{\B{E}[\exp(-sf)]}.
\]
Write $\B{E}_s$ for expectation under $\nu_s$. Boundedness of $f$ justifies differentiation under the expectation, giving
\[
\phi'(s)=-\B{E}_s[f],
\qquad
\phi''(s)=\operatorname{Var}_s(f).
\]
Since $\phi(0)=0$, Taylor's theorem with integral remainder, expanded about
$s=1$ and evaluated at $s=0$, gives
\[
0
= \phi(0)
= \phi(1)-\phi'(1)+\int_0^1 s\phi''(s)\,ds,
\]
and hence
\[
\phi'(1)-\phi(1)
=
\int_0^1 s\phi''(s)\,ds.
\]
Therefore, the expression in the second claim equals
\begin{align*}
-\frac{\B{E}[f\exp(-f)]}{\B{E}[\exp(-f)]}
-\log\B{E}[\exp(-f)]
=\phi'(1)-\phi(1)
=\int_0^1 s\phi''(s)\,ds
&=\int_0^1 s\operatorname{Var}_s(f)\,ds.
\end{align*}
Thus
\begin{align}
\label{eq:lem:reverse-jensen-gap:1}
-\frac{\B{E}[f\exp(-f)]}{\B{E}[\exp(-f)]}
-\log\B{E}[\exp(-f)]
&= \int_0^1 s\operatorname{Var}_s(f)\,ds
\le \int_0^1 s\B{E}_s[f^2]\,ds.
\end{align}

Next we claim that
\[
\operatorname{Cov}_s(f^2,f)
\geq
-M\B{E}_s[f^2].
\]
To see this, fix $s$ and let $F$ have the distribution of $f$ under $\nu_s$. Write
\[
\mu:=\B{E}_s[F],
\qquad
q:=\B{E}_s[F^2].
\]
If $\mu\leq0$, then $F^2(F+M)\geq0$ implies
\[
\operatorname{Cov}_s(F^2,F)+Mq
=
\B{E}_s[F^3]-\B{E}_s[F^2]\B{E}_s[F]+M\B{E}_s[F^2]
=
\B{E}_s[F^2(F+M)]-\mu q
\geq0.
\]
If $\mu>0$, the Cauchy--Schwarz inequality gives
\[
\B{E}_s[F^2(F+M)]
\geq
\frac{\B{E}_s[F(F+M)]^2}{\B{E}_s[F+M]}
=
\frac{(q+M\mu)^2}{\mu+M}.
\]
Moreover, $q\geq\mu^2$ by Jensen's inequality, and
\begin{align*}
(q+M\mu)^2-\mu q(\mu+M)
&=
q^2+2M\mu q+M^2\mu^2-\mu^2q-M\mu q
\\
&=
q(q-\mu^2)+M\mu q+M^2\mu^2
\geq0.
\end{align*}
Thus $\B{E}_s[F^2(F+M)]\geq\mu q$, which again gives
$\operatorname{Cov}_s(F^2,F)+Mq\geq0$ and proves the claim.

Differentiating the second moment under $\nu_s$ now gives
\begin{align*}
\frac{d}{ds}\B{E}_s[f^2]
&= \frac{d}{ds} \frac{\B{E}[f^2\exp(-sf)]}{\B{E}[\exp(-sf)]} \\
&= -\frac{\B{E}[f^3\exp(-sf)]}{\B{E}[\exp(-sf)]}
+
\frac{\B{E}[f^2\exp(-sf)]\B{E}[f\exp(-sf)]}
{\B{E}[\exp(-sf)]^2}
\\
&=
-\B{E}_s[f^3]+\B{E}_s[f^2]\B{E}_s[f]
\\
&=
-\operatorname{Cov}_s(f^2,f)
\leq
M\B{E}_s[f^2].
\end{align*}
Set $q(s):=\B{E}_s[f^2]$. Since $q'(s)\leq Mq(s)$ and
$q(0)=\B{E}[f^2]$, the integrating factor satisfies
\[
\frac{d}{ds}\bigl(\exp(-Ms)q(s)\bigr)
=
\exp(-Ms)\bigl(q'(s)-Mq(s)\bigr)
\leq0.
\]
Thus $\exp(-Ms)q(s)\leq q(0)$, and therefore
\[
\B{E}_s[f^2]
\leq
\exp(sM)\B{E}[f^2].
\]
Consequently, using Equation~\ref{eq:lem:reverse-jensen-gap:1}, we have
\begin{align*}
-\frac{\B{E}[f\exp(-f)]}{\B{E}[\exp(-f)]}
-\log\B{E}[\exp(-f)]
&\leq\int_0^1s\B{E}_s[f^2]\,ds \\
&\leq\B{E}[f^2]\int_0^1s\exp(sM)\,ds \\
&\leq\tfrac{1}{2}\exp(M)\B{E}[f^2],
\end{align*}
which proves the second claim when $f$ is bounded.

For a general measurable $f\geq-M$, the left-hand side is finite because
$\exp(-f)$ and $f\exp(-f)$ are bounded, and $\B{E}[\exp(-f)]>0$.
If $\B{E}[f^2]=\infty$, the claim is therefore immediate. Otherwise, let
$f_n:=f\wedge n$. Each $f_n$ is
bounded and satisfies $f_n\geq-M$, so the bounded case gives
\[
-\frac{\B{E}[f_n\exp(-f_n)]}{\B{E}[\exp(-f_n)]}
-\log\B{E}[\exp(-f_n)]
\leq \tfrac{1}{2}\exp(M)\B{E}[f_n^2].
\]
Since $f_n\to f$, the functions $\exp(-f_n)$ and
$f_n\exp(-f_n)$ are uniformly bounded, while $f_n^2\leq f^2$.
Dominated convergence therefore gives
\[
\B{E}[\exp(-f_n)]\to\B{E}[\exp(-f)],
\qquad
\B{E}[f_n\exp(-f_n)]\to\B{E}[f\exp(-f)],
\qquad
\B{E}[f_n^2]\to\B{E}[f^2].
\]
Since the limiting denominator is positive, passing to the limit proves the
second claim in general.
\end{proof}

\begin{lemma}\label{lem:frequentist-kl-bound}
Let $(\C{X}, \nu)$ be a probability space, $f_1, f_2 : \C{X} \to \B{R}$ be measurable and let $\pi_1 = S_\beta(f_1)$ and $\pi_2 = S_\beta(f_2)$ be the corresponding Gibbs distributions with parameter $\beta > 0$.
\begin{enumerate}
\item[(i)] If $f_1 \geq f_2$, then 
\begin{align*}
\max\{ D_{\mathrm{KL}}(\pi_1 \| \pi_2), D_{\mathrm{KL}}(\pi_2 \| \pi_1) \}
\leq \beta \B{E}_{\pi_1} \left[ f_1 - f_2 \right].
\end{align*}
\item[(ii)] If there is $M \geq 0$ such that $f_1 - f_2 \geq -M$, then 
\begin{align*}
\max\{ D_{\mathrm{KL}}(\pi_1 \| \pi_2), D_{\mathrm{KL}}(\pi_2 \| \pi_1) \}
\leq \tfrac{1}{2} \beta^2 \exp(\beta M) \B{E}_{\pi_1} \left[ (f_1 - f_2)^2 \right].
\end{align*}
\item[(iii)] If there is $M \geq 0$ such that $|f_1-f_2|\le M$, then
\[
\max\{ D_{\mathrm{KL}}(\pi_1 \| \pi_2), D_{\mathrm{KL}}(\pi_2 \| \pi_1) \}
\leq 2\beta M.
\]
\end{enumerate}
\end{lemma}

\begin{proof}
Let $g = f_1 - f_2$.
We have
\begin{align*}
\log\left( \frac{d \pi_1}{d \pi_2}(s) \right)
&= \log\left( \frac{d S_\beta(f_1)}{d S_\beta(f_2)} (s) \right) \\
&= \log\left( \frac{\exp(\beta f_1(s))}{\exp(\beta f_2(s))} \right)
+ \log\left( \frac{\int \exp(\beta f_2(s')) d \nu(s')}{\int \exp(\beta f_1(s')) d \nu(s')} \right) \\
&= \beta (f_1(s) - f_2(s))
+ \log\left( \frac{\int \exp(\beta (f_2(s') - f_1(s'))) \exp(\beta f_1(s')) d \nu(s')}{\int \exp(\beta f_1(s')) d \nu(s')} \right) \\
&= \beta g(s)
+ \log\left( \int \exp(- \beta g(s')) S_\beta(f_1)(s') d \nu(s') \right) \\
&= \beta g(s) + \log\left( \B{E}_{\pi_1} \left[ \exp(- \beta g) \right] \right).
\end{align*}
Therefore
\begin{align}\label{eq:lem:frequentist-kl-bound:1}
D_{\mathrm{KL}}(\pi_1 \| \pi_2)
&= \B{E}_{\pi_1}\left[ \log\left( \frac{d \pi_1}{d \pi_2} (s) \right) \right]
= \B{E}_{\pi_1}\left[ \beta g(s) \right] + \log\left( \B{E}_{\pi_1} \left[ \exp(- \beta g) \right] \right).
\end{align}
For the opposite direction, similarly,
\begin{align*}
\log\left(\frac{d\pi_2}{d\pi_1}(s)\right)
= - \log\left(\frac{d\pi_1}{d\pi_2}(s)\right) 
&= -\beta g(s)
- \log\left(\B{E}_{\pi_1}[\exp(-\beta g)]\right).
\end{align*}
Therefore,
\begin{align*}
D_{\mathrm{KL}}(\pi_2\|\pi_1)
&=
\B{E}_{\pi_2}\left[
\log\left(\frac{d\pi_2}{d\pi_1}(s)\right)
\right]
=
-\beta\B{E}_{\pi_2}[g]
-\log\left(\B{E}_{\pi_1}[\exp(-\beta g)]\right).
\end{align*}
Using the density of $\pi_2$ with respect to $\pi_1$, we also have
\begin{align*}
\B{E}_{\pi_2}[g]
&=
\B{E}_{\pi_1}\left[
g\frac{d\pi_2}{d\pi_1}
\right]
=
\frac{\B{E}_{\pi_1}[g\exp(-\beta g)]}
{\B{E}_{\pi_1}[\exp(-\beta g)]}.
\end{align*}
Substituting this identity yields
\begin{align}\label{eq:lem:frequentist-kl-bound:2}
D_{\mathrm{KL}}(\pi_2\|\pi_1)
&=
-\frac{\B{E}_{\pi_1}[\beta g\exp(-\beta g)]}
{\B{E}_{\pi_1}[\exp(-\beta g)]}
-\log\left(\B{E}_{\pi_1}[\exp(-\beta g)]\right).
\end{align}
Now (i) and (ii) follow directly from Equations~\ref{eq:lem:frequentist-kl-bound:1} and~\ref{eq:lem:frequentist-kl-bound:2} and Lemmas~\ref{lem:jensen-gap} and~\ref{lem:reverse-jensen-gap}.
For (iii), the assumption $|g|\leq M$ gives
\[
\exp(-\beta M)
\leq
\B{E}_{\pi_1}[\exp(-\beta g)]
\leq
\exp(\beta M).
\]
Together with the two log-density-ratio identities above, this implies
\[
\left|\log\left(\frac{d\pi_1}{d\pi_2}(s)\right)\right|
=
\left|\log\left(\frac{d\pi_2}{d\pi_1}(s)\right)\right|
\leq 2\beta M.
\]
Taking the corresponding expectations proves (iii).
\end{proof}

\section{Bayesian bounds on KL-divergence}

\begin{lemma}\label{lem:ts-kl-bound-1}
Let $(\C{X}, \nu)$ be a measure space and let $\C{G}$ be a class of real valued bounded measurable functions over $\C{X}$ and assume $\varrho$ is a probability measure over $\C{G}$.
Then
\begin{align*}
\B{E}_{f_1, f_2 \sim \varrho} \left[ 
D_{\mathrm{KL}}(S_{\beta}(f_1) \| S_{\beta}(f_2))
\right]
&= \beta \B{E}_{f_1, f_2 \sim \varrho} \left[ 
\B{E}_{s \sim S_{\beta}(f_1)} \left[ 
f_1(s) - f_2(s)
\right]
\right] \\
&\leq 
\beta \B{E}_{f \sim \varrho} \left[
  \B{E}_{s \sim S_{\beta}(f)} \left[ w(s) \right]
\right],
\end{align*}
where $w(s) := \sup_{f \in \C{G}} f(s) - \inf_{f \in \C{G}} f(s)$.
\end{lemma}

\begin{proof}
We have
\begin{align*}
D_{\mathrm{KL}}(S_{\beta}(f_1) \| S_{\beta}(f_2))
&= \B{E}_{S_{\beta}(f_1)} \left[ 
\log\left( \frac{d S_{\beta}(f_1)}{d S_{\beta}(f_2)} (s) \right)
\right] \\
&= \B{E}_{S_{\beta}(f_1)} \left[ 
\beta(f_1(s) - f_2(s)) - \log Z_\beta(f_1) + \log Z_\beta(f_2)
\right].
\end{align*}
Since $f_2$ and $f_1$ are both sampled according to the same distribution, namely $\varrho$, we have 
\begin{align*}
\B{E}_{f_2, f_1 \sim \varrho} \left[ - \log Z_\beta(f_1) + \log Z_\beta(f_2) \right] = 0.
\end{align*}
Therefore
\begin{align*}
\B{E}_{f_2, f_1 \sim \varrho} \left[ 
D_{\mathrm{KL}}(S_{\beta}(f_1) \| S_{\beta}(f_2))
\right]
&= \B{E}_{f_2, f_1 \sim \varrho} \left[ 
\B{E}_{S_{\beta}(f_1)} \left[ 
\beta(f_1(s) - f_2(s)) - \log Z_\beta(f_1) + \log Z_\beta(f_2)
\right]
\right] \\
&= \beta \B{E}_{f_1, f_2 \sim \varrho} \left[ 
\B{E}_{s \sim S_{\beta}(f_1)} \left[ 
f_1(s) - f_2(s)
\right]
\right] \\
&\leq \beta \B{E}_{f \sim \varrho} \left[ 
\B{E}_{s \sim S_{\beta}(f)} \left[ 
w(s)
\right]
\right],
\end{align*}
where the inequality follows from the fact that $f_1(s) - f_2(s) \leq w(s)$.
\end{proof}

\begin{lemma}\label{lem:ts-kl-bound-2}
Let $(\C{X}, \nu)$ be a measure space and let $\C{G}$ be a class of real valued bounded measurable functions over $\C{X}$.
Define the function $w(s) := \sup_{f \in \C{G}} f(s) - \inf_{f \in \C{G}} f(s)$.
Then, for any $f_1, f_2 \in \C{G}$, we have
\begin{align*}
D_{\mathrm{KL}}(\pi_1 \| \pi_2)
+ D_{\mathrm{KL}}(\pi_2 \| \pi_1)
&\leq \beta^2 \left( \B{E}_{s \sim \pi_1} \left[ (w(s))^2 \right] 
+ \B{E}_{s \sim \pi_2} \left[ (w(s))^2 \right] 
\right).
\end{align*}
where $\pi_1 := S_\beta(f_1)$ and $\pi_2 := S_\beta(f_2)$.
Moreover, if $\varrho$ is a probability measure over $\C{G}$, then we have
\begin{align*}
\mathbb{E}_{f_1,f_2 \sim \varrho} D_{\mathrm{KL}}(S_\beta(f_1)\|S_\beta(f_2))
\le
\beta^2
\B{E}_{f \sim \varrho}
\left[
  \B{E}_{s \sim S_\beta(f)}
  \left[w(s)^2\right]
\right].
\end{align*}
\end{lemma}

\begin{proof}
Fix $f_1,f_2 \in \C{G}$ and define
\[
g := f_2 - f_1
,\quad a := \frac{1}{\beta} \log \B{E}_{\pi_1}[ \exp(\beta g)]
,\quad u := \beta(g-a).
\]
Then, as we saw in the proof of Lemma~\ref{lem:frequentist-kl-bound}, we have
\begin{align*}
\log\left( \frac{d \pi_1}{d \pi_2} (s) \right)
&= -\beta g(s) + \log\left( \B{E}_{\pi_1} \left[ \exp(\beta g) \right] \right)
= -\beta g(s) + \beta a = -u.
\end{align*}
We compute
\[
D_{\mathrm{KL}}(\pi_1\|\pi_2) = \mathbb E_{\pi_1} \left[ \log\frac{d\pi_1}{d\pi_2} \right] = -\int u\, d\pi_1,
\qquad
D_{\mathrm{KL}}(\pi_2\|\pi_1) = \mathbb E_{\pi_2} \left[  \log\frac{d\pi_2}{d\pi_1} \right] = \int u\, d\pi_2.
\]
Hence
\[
D_{\mathrm{KL}}(\pi_1\|\pi_2) + D_{\mathrm{KL}}(\pi_2\|\pi_1)
=
\int (e^u - 1) u \, d\pi_1.
\]
Using the inequality
\[
(e^u - 1)u \le \frac{e^u + 1}{2} u^2 \quad \text{for all } u \in \mathbb{R},
\]
we obtain
\[
D_{\mathrm{KL}}(\pi_1\|\pi_2) + D_{\mathrm{KL}}(\pi_2\|\pi_1)
\le
\frac{1}{2} \int (e^u + 1) u^2 \, d\pi_1.
\]
Since $d\pi_2 = e^u d\pi_1$, this becomes
\[
\frac{1}{2} \left( \int u^2 \, d\pi_1 + \int u^2 \, d\pi_2 \right)
=
\frac{\beta^2}{2}
\left(
\int (g-a)^2 \, d\pi_1 + \int (g-a)^2 \, d\pi_2
\right).
\]

Let
\[
m_1 := \int g\, d\pi_1, \qquad m_2 := \int g\, d\pi_2.
\]
Define
\[
\psi(t) := \log \int e^{t g}\, d\pi_1.
\]
Then $\psi$ is convex, $\psi'(0) = m_1$, $\psi'(\beta) = m_2$, and
\[
a = \frac{\psi(\beta) - \psi(0)}{\beta}.
\]
Thus
\[
m_1 \le a \le m_2.
\]

Define
\[
F(c) := \int (g-c)^2 \, d\pi_1 + \int (g-c)^2 \, d\pi_2.
\]
Since $F$ is convex in $c$ and $a \in [m_1, m_2]$,
\[
F(a) \le \max\{F(m_1), F(m_2)\}.
\]
A direct computation gives
\[
F(m_1) = \int g^2 \, d\pi_1 + \int g^2 \, d\pi_2 - 2 m_1 m_2,
\]
and similarly for $F(m_2)$. Hence
\[
F(a) \le \int g^2 \, d\pi_1 + \int g^2 \, d\pi_2 - 2 m_1 m_2.
\]
Using
\[
-2 m_1 m_2 \le m_1^2 + m_2^2
\le \int g^2 \, d\pi_1 + \int g^2 \, d\pi_2,
\]
we obtain
\[
\int (g-a)^2 \, d\pi_1 + \int (g-a)^2 \, d\pi_2
\le
2\left( \int g^2 \, d\pi_1 + \int g^2 \, d\pi_2 \right).
\]

Therefore
\[
D_{\mathrm{KL}}(\pi_1\|\pi_2) + D_{\mathrm{KL}}(\pi_2\|\pi_1)
\le
\beta^2
\left(
\int g^2 \, d\pi_1 + \int g^2 \, d\pi_2
\right).
\]
Substituting $g = f_2 - f_1$,
\begin{align*}
D_{\mathrm{KL}}(\pi_1 \| \pi_2) + D_{\mathrm{KL}}(\pi_2 \| \pi_1)
&\le
\beta^2
\left(
\int (f_1-f_2)^2 \, d\pi_1
+
\int (f_1-f_2)^2 \, d\pi_2
\right) \\
&\le
\beta^2
\left(
\int w^2 \, d\pi_1
+
\int w^2 \, d\pi_2
\right).
\end{align*}

Now take expectation over independent $f_1,f_2 \sim \varrho$. Since $w$ is fixed by $\C{G}$ and both functions have the same law,
\[
\mathbb E_{f_1,f_2}\!\left[\int w^2\,dS_\beta(f_1)\right]
=
\mathbb E_{f_1,f_2}\!\left[\int w^2\,dS_\beta(f_2)\right]
=
\mathbb E_{f\sim\varrho}\!\left[\mathbb E_{s\sim S_\beta(f)}[w(s)^2]\right].
\]
Both width terms therefore reduce to the same expectation under the mixture obtained by drawing $f\sim\varrho$ and then sampling from $S_\beta(f)$. Exchangeability also gives
\[
\mathbb{E}_{f_1,f_2} D_{\mathrm{KL}}(S_\beta(f_1)\|S_\beta(f_2))
=
\mathbb{E}_{f_1,f_2} D_{\mathrm{KL}}(S_\beta(f_2)\|S_\beta(f_1)).
\]
Thus
\[
2\,\mathbb{E}_{f_1,f_2} D_{\mathrm{KL}}(S_\beta(f_1)\|S_\beta(f_2))
\le
2\beta^2 \, \mathbb{E}_{f_1,f_2} \int w^2 \, dS_\beta(f_1)
=
2\beta^2
\B{E}_{f \sim \varrho} \left[
  \B{E}_{s \sim S_{\beta}(f)} \left[ w(s)^2 \right]
\right],
\]
and dividing by $2$ yields the stated bound under the Gibbs mixture, with coefficient $\beta^2$.
\end{proof}

\section{Proof of Theorem~\ref{thm:regret-decom-swc}}
\label{apx:regret-decom-swc}

\begin{proof}
Let $\mathcal E_t = \{f_\theta \in \mathcal F_t\}$. Since
$0\leq L_t\leq g_t\leq U_t\leq C$ and $0\leq f_\theta\leq C$, we have
$|g_t-f_\theta|\leq C$. Thus Lemma~\ref{lem:frequentist-kl-bound}(iii)
bounds either directed instantaneous KL regret by $2C$.

On $\mathcal E_t$, we have $g_t \geq L_t \geq 0$ and $f_\theta \leq U_t \leq C$. Therefore $g_t - f_\theta \geq -C$ and we may use Lemma~\ref{lem:frequentist-kl-bound}(ii) to see that
\begin{align*}
\frac{1}{\beta}
\max\!\left\{
D_{\mathrm{KL}}\big(S_{\beta,t}(g_t)\|S_{\beta,t}(f_\theta)\big),
D_{\mathrm{KL}}\big(S_{\beta,t}(f_\theta)\|S_{\beta,t}(g_t)\big)
\right\}
\leq 
\tfrac{1}{2} \beta \exp(\beta C)
\mathbb E\!\left[
w_{\mathcal F_t}(A_t)^2
\mid H_t,\mathcal A_t
\right].
\end{align*}
Taking expectation and summing over $t=1,\dots,T$ completes the proof.
\end{proof}

\section{Proof of Theorem~\ref{thm:regret-decom-ucs}}
\label{apx:regret-decom-ucs}

\begin{proof}
Let $\mathcal E_t = \{f_\theta \in \mathcal F_t\}$. On $\mathcal E_t$,
optimism and the confidence bounds give
$0 \leq g_t-f_\theta = U_t-f_\theta \leq w_{\mathcal F_t}$ on $\mathcal A_t$.
Applying Lemma~\ref{lem:frequentist-kl-bound} with $f_1=U_t$ and
$f_2=f_\theta$, part~(i) yields the linear-width bound. The nonnegative gap
allows $M=0$ in part~(ii), so $\exp(\beta M)=1$ and the squared-width bound
follows. Thus, on $\mathcal E_t$,
\begin{align*}
\tfrac{1}{\beta}
\max\!\left\{
D_{\mathrm{KL}}\big(S_{\beta,t}(g_t)\|S_{\beta,t}(f_\theta)\big),
D_{\mathrm{KL}}\big(S_{\beta,t}(f_\theta)\|S_{\beta,t}(g_t)\big)
\right\}
&\leq
\mathbb E\!\left[
w_{\mathcal F_t}(A_t) \mid H_t,\mathcal A_t
\right],\\
\tfrac{1}{\beta}
\max\!\left\{
D_{\mathrm{KL}}\big(S_{\beta,t}(g_t)\|S_{\beta,t}(f_\theta)\big),
D_{\mathrm{KL}}\big(S_{\beta,t}(f_\theta)\|S_{\beta,t}(g_t)\big)
\right\}
&\leq
\tfrac{1}{2}\beta
\mathbb E\!\left[
(w_{\mathcal F_t}(A_t))^2 \mid H_t,\mathcal A_t
\right].
\end{align*}
Since $0\leq U_t,f_\theta\leq C$, on $\mathcal E_t^c$ either directed KL
regret is at most $2C$ by Lemma~\ref{lem:frequentist-kl-bound}(iii).
Taking expectations over all randomness and summing over
$t=1,\dots,T$ completes the proof.
\end{proof}

\section{Proof of Theorem~\ref{thm:regret-decom-ts2}}
\label{apx:regret-decom-ts2}

\begin{proof}
Fix $t \geq 1$.
Let $\rho_t^1$ denote the posterior $\rho(\cdot \mid H_t,\mathcal A_t)$.
Then, since TS$^2$ samples from the posterior, conditioned on $(H_t,\mathcal A_t)$ the distribution of $\tilde\theta_t$ is also $\rho_t^1$.
Let $\rho_t^2$ denote $\rho_t^1$ conditioned on $f_\theta \in \mathcal F_t$.
Since $\theta$ and $\tilde\theta_t$ are conditionally i.i.d.\ given $(H_t,\mathcal A_t)$, $\rho_t^2$ is also the conditional law of $\tilde\theta_t$ given $g_t \in \mathcal F_t$.
Let $\mathcal E_t$ and $\mathcal E'_t$ denote the events $f_\theta \in \mathcal F_t$ and $g_t \in \mathcal F_t$, respectively, and define
$\mathcal G_t := \mathcal E_t \cap \mathcal E'_t$.
Given $(H_t,\mathcal A_t,\mathcal G_t)$, the parameters remain independent with common law $\rho_t^2$: each confidence restriction conditions one of the two independent draws. We work on histories with $\mathbb P(\mathcal G_t\mid H_t,\mathcal A_t)>0$; when this conditional probability is zero, the good-event contribution vanishes.
Therefore, using Lemma~\ref{lem:ts-kl-bound-1}, we have
\begin{align*}
\tfrac{1}{\beta}
\mathbb{E}_{\theta, \tilde\theta_t \sim \rho_t^2}
D_{\mathrm{KL}}\big(S_{\beta,t}(g_t)\|S_{\beta,t}(f_\theta)\big)
&\leq 
\mathbb E_{\tilde\theta_t \sim \rho_t^2} \left[
  \mathbb E_{A \sim S_{\beta,t}(g_t)} \left[ w_{\mathcal F_t}(A) \right]
\right] 
=
\mathbb E\!\left[ w_{\mathcal F_t}(A_t) \,\middle\vert\, H_t, \mathcal A_t, \mathcal G_t \right]
\end{align*}
Applying the posterior-averaged part of Lemma~\ref{lem:ts-kl-bound-2} to the function law induced by $\rho_t^2$ gives the squared-width bound. The Gibbs mixture in the lemma samples $\tilde\theta_t\sim\rho_t^2$ and then $A\sim S_{\beta,t}(g_t)$, which is the action law of TS$^2$ conditional on $(H_t,\mathcal A_t,\mathcal G_t)$. Hence
\begin{align*}
\tfrac{1}{\beta}
\mathbb{E}_{\theta, \tilde\theta_t \sim \rho_t^2}
D_{\mathrm{KL}}\big(S_{\beta,t}(g_t)\|S_{\beta,t}(f_\theta)\big)
&\leq 
\beta \mathbb E_{\tilde\theta_t \sim \rho_t^2} \left[
  \mathbb E_{A \sim S_{\beta,t}(g_t)} \left[ (w_{\mathcal F_t}(A))^2 \right]
\right] \\
&=
\beta \mathbb E\!\left[ (w_{\mathcal F_t}(A_t))^2 \,\middle\vert\, H_t, \mathcal A_t, \mathcal G_t \right].
\end{align*}
Exchanging $\theta$ and $\tilde\theta_t$ leaves the expected KL on the left unchanged. The forward-KL bound therefore has the same constants, confidence event $\mathcal G_t$, and width expectations.

Since $g_t$ is sampled from the posterior and $f_\theta$ has the same conditional law given $(H_t,\mathcal A_t)$, we have
\begin{align*}
\mathbb P(\mathcal E_t^c \mid H_t,\mathcal A_t)
=
\mathbb P((\mathcal E'_t)^c \mid H_t,\mathcal A_t).
\end{align*}
Therefore,
\begin{align*}
\mathbb P(\mathcal G_t^c \mid H_t,\mathcal A_t)
&\le
\mathbb P(\mathcal E_t^c \mid H_t,\mathcal A_t)
+
\mathbb P((\mathcal E'_t)^c \mid H_t,\mathcal A_t) 
=
2\,\mathbb P(\mathcal E_t^c \mid H_t,\mathcal A_t).
\end{align*}
Taking expectations yields
\begin{align*}
\mathbb P(\mathcal G_t^c)
\le
2\,\mathbb P(\mathcal E_t^c)
\le
2\,\mathbb P(\mathcal E^c),
\end{align*}
since $\mathcal E^c = \bigcup_{s=1}^T \mathcal E_s^c$ and hence $\mathcal E_t^c \subseteq \mathcal E^c$.

The following argument applies to either KL direction; we display the reverse
direction. On $\mathcal G_t$, the per-round regret is bounded by the
conditional expectation above. Since $g_t$ and $f_\theta$ take values in
$[0,C]$, on $\mathcal G_t^c$ either directed instantaneous KL regret is at most $2C$ by
Lemma~\ref{lem:frequentist-kl-bound}(iii). Therefore,
\begin{align*}
\tfrac{1}{\beta}
\mathbb E\!\left[
D_{\mathrm{KL}}\big(S_{\beta,t}(g_t)\|S_{\beta,t}(f_\theta)\big)
\mid H_t,\mathcal A_t
\right]
&\le
\mathbb E\!\left[
w_{\mathcal F_t}(A_t)\mathbf 1_{\mathcal G_t}
\mid H_t,\mathcal A_t
\right]
+
2C\,\mathbb P(\mathcal G_t^c\mid H_t,\mathcal A_t) \\
&\le
\mathbb E\!\left[
w_{\mathcal F_t}(A_t)
\mid H_t,\mathcal A_t
\right]
+
2C\,\mathbb P(\mathcal G_t^c\mid H_t,\mathcal A_t).
\end{align*}
Similarly,
\begin{align*}
\tfrac{1}{\beta}
\mathbb E\!\left[
D_{\mathrm{KL}}\big(S_{\beta,t}(g_t)\|S_{\beta,t}(f_\theta)\big)
\mid H_t,\mathcal A_t
\right]
&\le
\beta\,
\mathbb E\!\left[
w_{\mathcal F_t}(A_t)^2\mathbf 1_{\mathcal G_t}
\mid H_t,\mathcal A_t
\right]
+
2C\,\mathbb P(\mathcal G_t^c\mid H_t,\mathcal A_t) \\
&\le
\beta\,
\mathbb E\!\left[
w_{\mathcal F_t}(A_t)^2
\mid H_t,\mathcal A_t
\right]
+
2C\,\mathbb P(\mathcal G_t^c\mid H_t,\mathcal A_t).
\end{align*}
Taking expectations over all randomness and summing over $t=1,\dots,T$ completes the proof.
\end{proof}

\section{Bounding width process with squared Eluder coefficient}
\label{apx:width-eluder}

\begin{proof}[Proof of Proposition~\ref{prop:width-eluder}]
The geometry of these confidence sets directly constrains pairwise disagreements.
Indeed, for any $f,g\in\mathcal F_t$,
\[
\|f-g\|_{2,E_t}
\le
\|f-\hat f_t^{\mathrm{LS}}\|_{2,E_t}
+
\|g-\hat f_t^{\mathrm{LS}}\|_{2,E_t}
\le
2\sqrt{\gamma_t},
\]
and therefore
\[
\sum_{s<t}(f(A_s)-g(A_s))^2
\le
4\gamma_t.
\]
Hence,
\[
|f(a)-g(a)|
=
\frac{|f(a)-g(a)|}
{\sqrt{\epsilon+\sum_{s<t}(f(A_s)-g(A_s))^2}}
\sqrt{\epsilon+\sum_{s<t}(f(A_s)-g(A_s))^2}
\]
\[
\le
\sqrt{\epsilon+4\gamma_t}
\frac{|f(a)-g(a)|}
{\sqrt{\epsilon+\sum_{s<t}(f(A_s)-g(A_s))^2}}.
\]
Taking the supremum over $f,g\in\mathcal F_t$ on the left-hand side and over $f,g\in\mathcal F$ on the right-hand side yields
\[
w_{\mathcal F_t}(a)
\le
\sqrt{\epsilon+4\gamma_t} \sigma_t(a)
,
\]
as claimed.
To prove the second claim, we note that $\gamma_t$ is increasing and, for every realized action sequence $A_1,\dots,A_T$, the definition directly implies that
\[
\sum_{t=1}^T \sigma_t(A_t)^2
\le
\operatorname{EC}^{\operatorname{sq}}(\epsilon,\mathcal F,T),
\]
Hence, by Cauchy–Schwarz inequality,
\[
\sum_{t=1}^T \sigma_t(A_t)
\le
\sqrt{\left(\sum_{t=1}^T \sigma_t(A_t)^2\right) T}
\le
\sqrt{\operatorname{EC}^{\operatorname{sq}}(\epsilon,\mathcal F,T) T}.
\qedhere
\]
\end{proof}

\section{Eluder dimension vs. squared eluder coefficient}
\label{apx:eluder}

\begin{lemma}
\label{lem:ecsq-eluder}
Assume $\mathcal F \subseteq \{f:\mathcal A\to[0,C]\}$.
Fix $\epsilon > 0$, $\theta \in (0, 1/2]$, and define
\[
    M
    :=
    \lceil \theta \log_2(T) \rceil
    ,
    \quad
    \alpha_k := 2^{-k}C,
\]
for $k=0,\dots,M+1$.
Then, for $\epsilon > 0$ and $T \geq \max\{ 32, \epsilon C^{-2} \}$, we have
\[
    \operatorname{EC}^{\operatorname{sq}}(\epsilon,\mathcal F,T)
    \le
    \frac{C^2}{\epsilon} T^{1-2\theta}
    +
    \sum_{k=0}^{M}
    \left(
    36 \log_2(T)
    +
    \frac{4\alpha_{k+1}^2}{\epsilon}
    \right)
    \dim_E(\mathcal F,\alpha_{k+1}).
\]
\end{lemma}

\begin{proof}
Fix an arbitrary action sequence $a_1,\dots,a_T$ and arbitrary function pairs
$(f_t,g_t)_{t=1}^T$ with $f_t,g_t\in\mathcal F$.

Define
\[
    \alpha_{\min} := C T^{-\theta} \in (0, C)
    ,\quad
    \Delta_t(a):=f_t(a)-g_t(a)
    ,\quad
    P_t:=\sum_{s<t}\Delta_t(a_s)^2.
\]
We need to bound
\[
    S
    :=
    \sum_{t=1}^T
    \frac{\Delta_t(a_t)^2}{\epsilon+P_t}.
\]

First consider rounds with small current disagreement:
\[
    \mathcal S_0
    :=
    \{t:|\Delta_t(a_t)|\le \alpha_{\min}\}.
\]
For these rounds,
\[
    \frac{\Delta_t(a_t)^2}{\epsilon+P_t}
    \le
    \frac{\alpha_{\min}^2}{\epsilon}.
\]
Therefore,
\[
    \sum_{t\in\mathcal S_0}
    \frac{\Delta_t(a_t)^2}{\epsilon+P_t}
    \le 
    \frac{T \alpha_{\min}^2}{\epsilon}
    = \frac{C^2}{\epsilon} T^{1-2\theta}.
\]

It remains to control the rounds with
$|\Delta_t(a_t)|>\alpha_{\min}$.
For $k=0,\dots,M$, define
\[
    \mathcal S_k
    :=
    \left\{
        t:
        \alpha_{k+1}
        <
        |\Delta_t(a_t)|
        \le
        \alpha_k
    \right\}.
\]
The sets $\mathcal S_k$ cover all remaining rounds.

Fix $k$. We further decompose $\mathcal S_k$ according to the size of the past disagreement.
Define the low-denominator set
\[
    \mathcal S_{k,-}
    :=
    \left\{
        t\in\mathcal S_k:
        \epsilon+P_t<\alpha_{k+1}^2
    \right\}.
\]
For $\ell\ge0$, define
\[
    \mathcal S_{k,\ell}
    :=
    \left\{
        t\in\mathcal S_k:
        2^\ell \alpha_{k+1}^2
        \le
        \epsilon+P_t
        <
        2^{\ell+1}\alpha_{k+1}^2
    \right\}.
\]

We first control $\mathcal S_{k,-}$.  Let
$t_1<\cdots<t_m$ be its elements.  For every $j\leq m$,
\[
    \sum_{i<j}
    \Delta_{t_j}(a_{t_i})^2
    \leq
    P_{t_j}
    <
    \alpha_{k+1}^2,
\]
while
\[
    |\Delta_{t_j}(a_{t_j})|
    >
    \alpha_{k+1}.
\]
Thus $a_{t_1},\dots,a_{t_m}$ is an
$\alpha_{k+1}$-eluder-independent sequence, and hence
\[
    |\mathcal S_{k,-}|
    \leq
    \dim_E(\mathcal F,\alpha_{k+1}).
\]
Moreover, for $t\in\mathcal S_{k,-}$,
\[
    \frac{\Delta_t(a_t)^2}{\epsilon+P_t}
    \leq
    \frac{\alpha_k^2}{\epsilon}
    =
    \frac{4\alpha_{k+1}^2}{\epsilon}.
\]
Therefore,
\[
    \sum_{t\in\mathcal S_{k,-}}
    \frac{\Delta_t(a_t)^2}{\epsilon+P_t}
    \leq
    \frac{4\alpha_{k+1}^2}{\epsilon}
    \dim_E(\mathcal F,\alpha_{k+1}).
\]

Since $f_t,g_t\in[0,C]$, we have $|\Delta_t(a)|\le C$ and hence
\[
    P_t\le TC^2.
\]
Therefore $\mathcal S_{k,\ell}$ is empty unless
\[
2^\ell \alpha_{k+1}^2 \leq \epsilon + TC^2,
\]
which implies that
\[
2 TC^2
\ge
\epsilon + TC^2
\ge 
2^\ell \alpha_{k+1}^2
= 2^{\ell - 2(k+1)} C^2
\ge 2^{\ell - 2(M+1)} C^2.
\]
Hence, using $T \geq 32$ and $M < \theta \log_2(T) + 1 \leq \tfrac{1}{2} \log_2(T) + 1$, we see that
\[
2^{\ell}
\le
2^{2 M + 3} T
\le
2^{\log_2(T) + 5} T
= 
32 T^{2}
\le 
T^{3}.
\]
Thus, for each $k$, no more than
\[
    L_k
    :=
    3 \log_2(T)
\]
values of $\ell$ can occur.

We now bound the size of $\mathcal S_{k,\ell}$ using only the ordinary eluder dimension.
Let
\[
    t_1<t_2<\cdots<t_m
\]
be the elements of $\mathcal S_{k,\ell}$.

For each $j\le m$, since $t_1,\dots,t_{j-1}<t_j$, we have
\[
    \sum_{i<j}\Delta_{t_j}(a_{t_i})^2
    \le
    \sum_{s<t_j}\Delta_{t_j}(a_s)^2
    =
    P_{t_j}
    <
    2^{\ell+1}\alpha_{k+1}^2.
\]
Also, since $t_j\in\mathcal S_k$,
\[
    |\Delta_{t_j}(a_{t_j})|>\alpha_{k+1}.
\]

We claim that the sequence
$a_{t_1},\dots,a_{t_m}$ can be partitioned into at most
\[
    q_\ell := 2^{\ell+1}+1
\]
subsequences, each of which is $\alpha_{k+1}$-eluder independent.

To see this, assign the points greedily to $q_\ell$ blocks. Suppose we are assigning $a_{t_j}$. For each block $B$, let
\[
    R_B(j)
    :=
    \sum_{i<j:\, a_{t_i}\in B}
    \Delta_{t_j}(a_{t_i})^2.
\]
If every block satisfied $R_B(j)>\alpha_{k+1}^2$, then summing over all blocks would give
\[
    \sum_{i<j}\Delta_{t_j}(a_{t_i})^2
    >
    q_\ell \alpha_{k+1}^2
    >
    2^{\ell+1}\alpha_{k+1}^2,
\]
contradicting the bound above. Hence there exists a block $B$ with
\[
    R_B(j)\le \alpha_{k+1}^2.
\]
Assign $a_{t_j}$ to such a block.

Now consider any one block. When a point $a_{t_j}$ is assigned to this block, we have
\[
    \sum_{i<j:\,a_{t_i}\in B}
    (f_{t_j}(a_{t_i})-g_{t_j}(a_{t_i}))^2
    \le
    \alpha_{k+1}^2,
\]
while
\[
    |f_{t_j}(a_{t_j})-g_{t_j}(a_{t_j})|
    =
    |\Delta_{t_j}(a_{t_j})|
    >
    \alpha_{k+1}.
\]
Therefore, each point in the block is $\alpha_{k+1}$-independent of its predecessors in that block. Thus the length of each block is at most
\[
    \dim_E(\mathcal F,\alpha_{k+1}).
\]
Since there are at most $q_\ell$ blocks,
\[
    |\mathcal S_{k,\ell}|
    \le
    q_\ell \dim_E(\mathcal F,\alpha_{k+1})
    \le
    (2^{\ell+1}+1) \dim_E(\mathcal F,\alpha_{k+1})
\]

For $t\in\mathcal S_{k,\ell}$,
\[
    \Delta_t(a_t)^2
    \le
    \alpha_k^2
    =
    4\alpha_{k+1}^2,
\]
and
\[
    \epsilon+P_t
    \ge
    2^\ell \alpha_{k+1}^2.
\]
Hence
\[
    \frac{\Delta_t(a_t)^2}{\epsilon+P_t}
    \le
    \frac{4}{2^\ell}.
\]
Therefore,
\[
    \sum_{t\in\mathcal S_{k,\ell}}
    \frac{\Delta_t(a_t)^2}{\epsilon+P_t}
    \le
    \frac{4}{2^\ell}
    |\mathcal S_{k,\ell}|
    \le
    12 \,
    \dim_E(\mathcal F,\alpha_{k+1}).
\]

The sets $\mathcal S_{k,-}$ and
$\mathcal S_{k,\ell}$, $\ell=0,\dots,\lfloor L_k\rfloor$,
partition $\mathcal S_k$.  Combining the low-denominator bound with the
sum over the dyadic levels gives
\begin{align*}
    \sum_{t\in\mathcal S_k}
    \frac{\Delta_t(a_t)^2}{\epsilon+P_t}
    &\le
    \left(
    36 \log_2(T)
    +
    \frac{4\alpha_{k+1}^2}{\epsilon}
    \right)
    \dim_E(\mathcal F,\alpha_{k+1}).
\end{align*}

Finally, summing over $k=0,\dots,M$ and adding the contribution of $\mathcal S_0$ yields
\[
    S
    \le
    \frac{C^2}{\epsilon} T^{1-2\theta}
    +
    \sum_{k=0}^{M}
    \left(
    36 \log_2(T)
    +
    \frac{4\alpha_{k+1}^2}{\epsilon}
    \right)
    \dim_E(\mathcal F,\alpha_{k+1}).
\]

Since the action sequence and function-pair sequence were arbitrary, the same bound holds for
$\operatorname{EC}^{\operatorname{sq}}(\epsilon,\mathcal F,T)$.

Since $k\geq0$, we have $\alpha_{k+1}\leq C/2$, and therefore
\[
    \frac{4\alpha_{k+1}^2}{\epsilon}
    \leq
    \frac{C^2}{\epsilon}.
\]
Moreover, if $\alpha_{k+1}^2\leq\epsilon$, then
$\mathcal S_{k,-}=\emptyset$, so the corresponding low-denominator term
vanishes entirely.
\end{proof}

\begin{corollary}
\label{cor:ecsq-vs-eluder-extended}
Suppose that for every $\varepsilon>0$,
\[
\dim_E(\mathcal F,\varepsilon)
\le
c(T)\,\varepsilon^{-a}\log_2(1/\varepsilon)^b
\]
for some constants $a,b\ge0$ and some non-negative function $c$.
Then, for all $\epsilon > 0$ and $T \ge \max\{32, 16 C^{-2}, \epsilon C^{-2}\}$, the following bounds hold.
\begin{enumerate}
\item[(i)] If $a>0$, then
\begin{align*}
\operatorname{EC}^{\operatorname{sq}}(\epsilon,\mathcal F,T)
&\le
\frac{C^2}{\epsilon}T^{a/(a+2)} \\
&\quad+
c(T)
\frac{2^{3a}}{2^a-1}
C^{-a}
T^{a/(a+2)}
\log_2(T)^b
\left(
36\log_2(T)
+
\frac{C^2}{\epsilon}
\right) \\
&= \tilde{O}\left(
\left(
\frac{1}{\epsilon}
+
c(T)\left(1+\frac{1}{\epsilon}\right)
\right)
T^{a/(a+2)}
\right).
\end{align*}

\item[(ii)]  If $a=0$, then
\begin{align*}
\operatorname{EC}^{\operatorname{sq}}(\epsilon,\mathcal F,T)
&\le
\frac{C^2}{\epsilon}
+
36 c(T) \log_2(T)^{b+2}
+
\frac{C^2}{\epsilon}
c(T)\log_2(T)^{b+1} \\
&= \tilde{O}\left(
\frac{1}{\epsilon}
+
c(T)\left(1+\frac{1}{\epsilon}\right)
\right).
\end{align*}
\end{enumerate}
\end{corollary}

\begin{proof}
By Lemma~\ref{lem:ecsq-eluder}, for any $0 < \theta \leq 1/2$, with
\[
M
= \lceil \theta\log_2(T)\rceil
\leq \left\lceil \frac{1}{2}\log_2(T) \right\rceil
< \frac{1}{2}\log_2(T) + 1
,
\qquad
\alpha_k=2^{-k}C,
\]
we have
\[
\operatorname{EC}^{\operatorname{sq}}(\epsilon,\mathcal F,T)
\le
\frac{C^2}{\epsilon}T^{1-2\theta}
+
\sum_{k=0}^{M}
\left(
36\log_2(T)
+
\frac{4\alpha_{k+1}^2}{\epsilon}
\right)
\dim_E(\mathcal F,\alpha_{k+1}).
\]

First suppose $a>0$ and choose
$\theta=\frac{1}{a+2}$.
Then
$1-2\theta=\frac{a}{a+2}$,
so the first term becomes
\[
\frac{C^2}{\epsilon}T^{a/(a+2)}.
\]

Next, since
\[
\alpha_{k+1}=2^{-(k+1)}C,
\]
the assumed Eluder-dimension growth gives
\[
\dim_E(\mathcal F,\alpha_{k+1})
\le
c(T) \,2^{a(k+1)}C^{-a}
\left(k+1+\log_2(1/C)\right)^b.
\]

We have $T \geq 16 C^{-2}$.
Therefore, for $k\le M$,
\[
k+1+\log_2(1/C)
\le
M+1+\log_2(1/C)
<
\frac{1}{2}\log_2(T) + 2 + \log_2(1/C)
\le
\log_2(T).
\]
Hence
\[
\sum_{k=0}^{M}
\dim_E(\mathcal F,\alpha_{k+1})
\le
c(T) C^{-a}
\log_2(T)^b
\sum_{k=0}^{M}2^{a(k+1)}.
\]

The geometric sum satisfies
\[
\sum_{k=0}^{M}2^{a(k+1)}
=
2^a\frac{2^{a(M+1)}-1}{2^a-1}
\le
\frac{2^a}{2^a-1}2^{a(M+1)}.
\]

Since
\[
M=\left\lceil \frac{1}{a+2}\log_2(T)\right\rceil,
\]
we have
\[
M+1
\le
\frac{1}{a+2}\log_2(T)+2,
\]
and therefore
\[
2^{a(M+1)}
\le
2^{2a}T^{a/(a+2)}.
\]

Combining these inequalities,
\[
\sum_{k=0}^{M}
\dim_E(\mathcal F,\alpha_{k+1})
\le
c(T) 
\frac{2^{3a}}{2^a-1}
C^{-a}
T^{a/(a+2)}
\log_2(T)^b.
\]
Since $4\alpha_{k+1}^2\leq C^2$, we also have
\[
4\sum_{k=0}^{M}
\alpha_{k+1}^2
\dim_E(\mathcal F,\alpha_{k+1})
\leq
C^2
\sum_{k=0}^{M}
\dim_E(\mathcal F,\alpha_{k+1}).
\]
Substituting this into Lemma~\ref{lem:ecsq-eluder} proves the $a>0$ bound.

Now suppose $a=0$ and choose
$\theta=\frac12$.
We have
\[
\dim_E(\mathcal F,\alpha_{k+1})
\le
c(T) 
\left(k+1+\log_2(1/C)\right)^b.
\]
As before,
\[
k+1+\log_2(1/C)
<
\log_2(T).
\]
Therefore,
\[
\sum_{k=0}^{M}
\dim_E(\mathcal F,\alpha_{k+1})
\le
c(T) (M+1)
\log_2(T)^b
<
c(T) 
\log_2(T)^{b+1}.
\]
Again using $4\alpha_{k+1}^2\leq C^2$ gives
\[
4\sum_{k=0}^{M}
\alpha_{k+1}^2
\dim_E(\mathcal F,\alpha_{k+1})
<
C^2 c(T)
\log_2(T)^{b+1}.
\]
Substituting this into Lemma~\ref{lem:ecsq-eluder} gives
\[
\operatorname{EC}^{\operatorname{sq}}(\epsilon,\mathcal F,T)
\le
\frac{C^2}{\epsilon}
+
36 c(T) 
\log_2(T)^{b+2}
+
\frac{C^2}{\epsilon}
c(T)\log_2(T)^{b+1}.
\]
This proves the claim.
\end{proof}

\clearpage

\section{Extended Comparison and Algorithms}

\begin{table}[h!]
\centering
\small
\renewcommand{\arraystretch}{1.15}
\begin{tabular}{@{}lll@{}}
\toprule
 & \textbf{Bandit Optimization} & \textbf{Tilted Sampling} \\
\midrule
Target              & $\arg\max_{a \in \mathcal{A}_t} f_\theta(a)$
                    & $S_{\beta,t}(f_\theta) \propto \exp(\beta f_\theta)\,\mu_t$ \\
Learner's choice    & action $A_t \in \mathcal{A}_t$
                    & distribution $A_t \sim S_{\beta,t}(g_t)$ \\
Performance metric  & $f_\theta(A_t^\star) - f_\theta(A_t)$
                    & $\tfrac{1}{\beta} D_{\mathrm{KL}}(S_{\beta,t}(g_t)\,\|\,S_{\beta,t}(f_\theta))$ \\
Geometry            & non-smooth (hard $\arg\max$)
                    & smooth (KL / Fisher) \\
Uncertainty order   & first-order $w_{\mathcal{F}_t}(A_t)$
                    & second-order $w_{\mathcal{F}_t}(A_t)^2$ \\
Width process       & $\sum_{t=1}^{T} w_{\mathcal{F}_t}(A_t)$
                    & $\sum_{t=1}^{T} w_{\mathcal{F}_t}(A_t)^2$ \\
Typical scaling     & $\widetilde{\mathcal{O}}(\sqrt{T})$
                    & $\widetilde{\mathcal{O}}(\log T)$ \quad (fixed $\beta$) \\
Optimistic algo.    & UCB
                    & UCS \\
Bayesian algo.      & TS
                    & TS$^2$ \\
Limit as $\beta \to \infty$
                    & ---
                    & recovers UCB and TS \\
\bottomrule
\end{tabular}
\vspace{1mm}
\caption{\textbf{Optimization vs.\ sampling.}
KL geometry smooths the objective, converting first-order uncertainty in
classical bandit optimization into second-order regret in tilted sampling.
The Gibbs operator $S_{\beta,t}$ interpolates between the two regimes:
as $\beta\to\infty$ the sampling targets concentrate on maximizers and
tilted sampling recovers bandit optimization under the regularity conditions discussed in Appendix~\ref{app:rl-limit-sampling}.}
\label{tbl:comparison-2}
\end{table}

\begin{algorithm2e}[h!]
\DontPrintSemicolon
\SetKwInOut{Input}{Input}
\SetKwInOut{Output}{Output}
\caption{Upper Confidence Sampling (UCS)}
\label{alg:ucs}
\Input{horizon $T$; inverse temperature $\beta>0$; base measure $\mu$;
       confidence-set rule $\mathcal{F}_{(\cdot)}$.}
Initialize history $H_1 \gets \emptyset$\;
\For{$t = 1, 2, \dots, T$}{
   Observe available action set $\mathcal{A}_t$\;
   Form conditioned base measure $\mu_t(\cdot) \gets \mu(\cdot \mid \mathcal{A}_t)$\;
   Construct confidence set $\mathcal{F}_t \gets \mathcal{F}_t(H_t,\mathcal{A}_t)$\;
   Compute upper envelope
   $\;U_t(a) \gets \sup_{f\in\mathcal{F}_t} f(a)\;$ for $a\in\mathcal{A}_t$\;
   Sample action
   $A_t \sim S_{\beta,t}(U_t)
   \;\propto\; \exp\!\bigl(\beta\, U_t(a)\bigr)\,\mu_t(da)$\;
   Observe reward $Y_t = f_\theta(A_t) + \varepsilon_t$\;
   Update history $H_{t+1} \gets H_t \cup \{(\mathcal{A}_t, A_t, Y_t)\}$\;
}
\end{algorithm2e}

\begin{algorithm2e}[h!]
\DontPrintSemicolon
\SetKwInOut{Input}{Input}
\SetKwInOut{Output}{Output}
\caption{Thompson Sampling for Sampling (TS$^2$)}
\label{alg:ts2}
\Input{horizon $T$; inverse temperature $\beta>0$; base measure $\mu$;
       prior $\rho$ on $\theta\in\Omega$.}
Initialize history $H_1 \gets \emptyset$\;
\For{$t = 1, 2, \dots, T$}{
   Observe available action set $\mathcal{A}_t$\;
   Form conditioned base measure $\mu_t(\cdot) \gets \mu(\cdot \mid \mathcal{A}_t)$\;
   Sample parameter from posterior
   $\;\tilde\theta_t \sim \rho(\,\cdot \mid H_t, \mathcal{A}_t)\;$\;
   Set $g_t \gets f_{\tilde\theta_t}$\;
   Sample action
   $A_t \sim S_{\beta,t}(f_{\tilde\theta_t})
   \;\propto\; \exp\!\bigl(\beta\, f_{\tilde\theta_t}(a)\bigr)\,\mu_t(da)$\;
   Observe reward $Y_t = f_\theta(A_t) + \varepsilon_t$\;
   Update history $H_{t+1} \gets H_t \cup \{(\mathcal{A}_t, A_t, Y_t)\}$\;
}
\end{algorithm2e}
 
\end{document}